\documentclass[a4paper, colorinlistoftodos]{article}

\usepackage[english]{babel}
\usepackage{csquotes}

\usepackage{amsmath}
\usepackage{mathrsfs}
\usepackage{mathtools}
\usepackage{multicol}
\usepackage{amssymb}
\usepackage{enumerate}
\usepackage{dsfont}
\usepackage{bm}
\usepackage{graphicx}
\usepackage{hyperref}
\usepackage{physics}
\usepackage{multicol}
\usepackage{xcolor}
\usepackage{centernot}
\usepackage{appendix}
\usepackage{authblk}
\usepackage[a4paper, margin=2.5cm]{geometry}

\newcommand{\sgn}[1]{\operatorname{sgn}\{#1\}}

\newcommand{\tcref}[1]{\text{\cref{#1}}}

\DeclarePairedDelimiterX\Set[1]\lbrace\rbrace{\def\given{\;\delimsize\vert\;\allowbreak}#1}

\usepackage{amsthm}

\newtheorem{theorem}{\textbf{Theorem}}[section]

\newtheorem{proposition}{\textbf{Proposition}}[section]

\newtheorem{lemma}{\textbf{Lemma}}[theorem]
\newtheorem*{remark}{\textbf{Remark}}

\newtheoremstyle{red}{}{}{\color{red}}{}{\color{red}\bfseries}{:}{ }{}
\theoremstyle{red}

\usepackage{cleveref}

\usepackage{crossreftools}
\crefname{assumptions}{assumptions}{assumptions}
\Crefname{assumptions}{Assumptions}{Assumptions}
\crefname{conjecture}{conjecture}{conjectures}
\Crefname{conjecture}{Conjecture}{Conjectures}

\newcommand{\R}{\mathbb{R}}

\newcommand{\N}{\mathbb{N}}

\DeclareMathOperator*{\argmin}{arg\,min}

\DeclareMathOperator*{\dom}{dom}

\DeclareMathOperator*{\diam}{diam}
\DeclareMathOperator*{\supp}{supp}

\newcommand{\HS}{\mathcal{H}}

\newcommand{\MC}{\mathcal{M}}

\newcommand{\PC}{\mathcal{P}}
\newcommand{\B}{\mathcal{B}}

\newcommand{\X}{\mathcal{X}}

\newcommand{\U}{\mathcal{U}}

\newcommand{\RC}{\mathcal{R}}

\usepackage{todonotes}
\numberwithin{equation}{section}

\usepackage{parskip}

\usepackage{lastpage}
\usepackage{fancyhdr}

\allowdisplaybreaks

\usepackage[
    backend=biber,natbib,
    style=apa,
    doi=false,
    isbn=false,
    url=false,
    eprint=false,
]{biblatex}

\makeatletter
\newrobustcmd*{\parentexttrack}[1]{%
  \begingroup
  \blx@blxinit
  \blx@setsfcodes
  \blx@bibopenparen#1\blx@bibcloseparen
  \endgroup}

\AtEveryCite{%
  }

\makeatother

\AtEveryBibitem{\clearfield{note}\clearfield{addendum}}
\AtEveryCitekey{\clearfield{note}\clearfield{addendum}}

\title{Learning a Sparse Representation of Barron Functions with the Inverse Scale Space Flow}
\author[1,*]{Tjeerd Jan Heeringa}
\author[2]{Tim Roith}
\author[1]{Christoph Brune}
\author[2,3]{Martin Burger}
\affil[1]{Mathematics of Imaging \& AI, University of Twente, Enschede, The Netherlands}
\affil[2]{Helmholtz Imaging, Deutsches Elektronen-Synchrotron DESY, Notkestr. 85, 22607 Hamburg, Germany}
\affil[3]{Fachbereich Mathematik, Universit\"at Hamburg, Bundesstr. 55, 20146 Hamburg, Germany}
\affil[*]{Corresponding author: t.j.heeringa@utwente.nl}
\date{2023-12-05}

\makeatletter
\let\blx@rerun@biber\relax
\makeatother

\begin{document}
\maketitle

\begin{abstract}
\noindent This paper presents a method for finding a sparse representation of Barron functions. Specifically, given an $L^2$ function $f$, the inverse scale space flow is used to find a sparse measure $\mu$ minimising the $L^2$ loss between the Barron function associated to the measure $\mu$ and the function $f$. The convergence properties of this method are analysed in an ideal setting and in the cases of measurement noise and sampling bias. In an ideal setting the objective decreases strictly monotone in time to a minimizer with $\mathcal{O}(1/t)$, and in the case of measurement noise or sampling bias the optimum is achieved up to a multiplicative or additive constant. This convergence is preserved on discretization of the parameter space, and the minimizers on increasingly fine discretizations converge to the optimum on the full parameter space.
\\ \\
\textbf{keywords: }Barron Space, Bregman Iterations, Sparse Neural Networks, Inverse Scale Space, Optimization
\end{abstract}

\section{Introduction}
Most neural networks contain a subnetwork with fewer parameters that performs equally well \parencite{ramanujan_whats_2020}, and some of these subnetworks have been found to generalise equally or even better than their dense counterparts \parencites{liu_intrinsically_2019}{liu_sparse_2021}. However, it is a priori hard to determine which parameters of the network will be part of the subnetwork. Hence, various approaches have been developed for finding well performing sparse neural network. They fall roughly in three categories. The first is to add a term to the loss or regularizer that promotes sparsity. An example of this would be LASSO, in which a $\ell^1$ regularizer is added \parencite{tibshirani_regression_1996}. The second is to train a network first and prune it afterwards, meaning weights are reduced with as little as possible influence on the performance \parencite{molchanov_pruning_2017}. The third is to start with a sparse architecture, and add or remove neurons during training \parencite{dai_nest_2018}.

One of the methods, which starts from a sparse architecture, is based on the Bregman iteration \parencite{osher_iterative_2005}. This method has been introduced and thoroughly analysed for imaging and compressed sensing \parencites{burger_error_2007}{yin_bregman_2008}{burger_adaptive_2012}. The method works in these settings by progressively adding more detail to the reconstructed images and signals, respectively. 
A limitation of the original method is that it requires that often requires the problem to be convex. However, adaptations of the method, e.g., the linearized variant in \cites{benning_choose_2021}{bungert_bregman_2021}, where the loss is replaced by a first order approximation, allows for a successful application to neural networks. A major success of this method is that it is able to find an auto-encoder without ever explicitly defining an auto-encoder like architecture \parencite{bungert_neural_2021}. This shows that it has major potential for automatic neural network architecture design tasks. 

\subsection{Related work}
Bregman iterations were introduced in \cite{osher_iterative_2005} and further developed and analysed in \cites{yin_bregman_2008}{bachmayr2009iterative}{cai2009linearized}{cai2009convergence}{yin2010analysis}{burger_error_2007}{burger_adaptive_2012}{benning2018modern} as an algorithm to solve sparsity promoting regularisation tasks in computer vision. Linearized Bregman iterations as introduced in \cites{cai2009linearized}{yin_bregman_2008} can be seen as a generalization of the mirror descent algorithm \parencites{nesterov1983method}{beck2003mirror} to the non-differentiable, convex case. More recently, variants of the original algorithm have been applied in the context of machine learning, see, e.g., \parencites{bungert_bregman_2021}{ bungert_neural_2021}{wang2023lifted}{wang2023lifted2}.

Bregman iterations are the implicit Euler discretization of an inverse scale space flow. Going to the continuous limit has helped to find easy implementations for relatively complex functionals like the total variation functional, and has helped to obtain well-justified and simple stopping criteria \parencite{burger_nonlinear_2006}. In the finite-dimensional case of sparse regularization (and further generalizations) an exact time discretization can be found, which leads to efficient methods \parencites{burger_adaptive_2012}{moeller2013multiscale}.
We refer to \cite{Benning_Burger_2018} for recent overview.

Similar to inverse scale space flow being the continuous limit of the Bregman iterations, we have that the Barron spaces are the continuous limit of shallow neural network. It was proven that Barron functions have bounded point evaluations \parencites{bartolucci_understanding_2023}{spek_duality_2023}, Barron functions can be approximated in $L^p$ with rate $O(m^{-1/p})$ \parencite{e_representation_2022}, Barron spaces have a representer theorem \parencite{parhi_banach_2021} and that Barron spaces are a kind of integral \textit{reproducing kernel Banach spaces} (RKBS), a Banach space analogue to \textit{reproducing kernel Hilbert spaces} (RKHS) \parencite{bartolucci_understanding_2023}. The spaces are parametrized by the activation function of the networks. The Barron spaces associated to most of the commonly used non-periodic activation are embedded in the Barron space with ReLU as activation function \parencite{heeringa_embeddings_2023}. This Barron space together with the Barron spaces associated to the RePU, the higher-order generalization of the ReLU, are strongly related to BV spaces \parencites{e_representation_2022}{parhi_banach_2021}.

A fundamental open question in machine learning is how to find the best function representing your data. For Barron spaces, this means finding the best measure $\mu$ representing the Barron function $f$. Since the relation between $\mu$ and $f$ is linear, this leads to a convex minimization problem. Based on an alternative representation of Barron functions in probability space, the authors in \cite{wojtowytsch_convergence_2020} formulated a Wasserstein gradient flow for this problem based on the ideas of \cite{chizat_global_2018}. Under several assumptions, including omnidirectional initial conditions and satisfying the Morse--Sard property, this leads to a unique solution $\pi$ \parencite{wojtowytsch_convergence_2020}. However, not all Barron functions satisfy the Morse--Sard property, placing a limit on the functions that can be represented with this approach \parencite{wojtowytsch_convergence_2020}. Although this unique solution $\pi$ represents the Barron function $f$, it is not necessarily the probability measure for $f$ with the smallest semi-norm. In order to find sparse neural networks, there is a need for a method that minimizes this semi-norm as well. 

\subsection{Our contribution}
In this work, we study the convergence and error analysis of finding the smallest measure $\mu$ such that the Barron function $K\mu$ is close to $f$ using the inverse scale space. This is the continuous and infinite dimensional version of finding a sparse shallow neural network approximating samples of $f$. 

In particular, we consider the minimisation problem
\begin{subequations}\label{eq:contribution_minimisation}
\begin{align}
    \mu^{\text{opt}} = &\argmin_{\mu^\dagger\in\MC(\Omega)}J(\mu^\dagger) \\
    \text{s.t. }&\mu^\dagger\in\argmin_{\mu\in\MC(\Omega)}\frac{1}{2}\norm{f-K\mu}^2_{L^2(\rho)}
\end{align}
\end{subequations}
where $J$ encodes the Barron norm and acts as regularizer and $L_\rho$ is the adjoint of $K$. In \cref{sec:main} we define these operators more rigorously, and show that the associated inverse scale space is given by
\begin{subequations}\label{eq:inverse_scale_space_contribution}
    \begin{align}
        \mu_t &= \argmin_{u\in \partial J^\ast(p_t)} \RC_f(\mu) & u_0 = 0,\\
        \partial_tp_t &= L_{\rho}(f-K\mu_t) & p_0 = 0.
    \end{align}
\end{subequations}
The data function $f$ and the data distribution $\rho$ are instance dependent, and the convergence behaviour and the error analysis of \cref{eq:inverse_scale_space_contribution} are dependent on these. In machine learning, measurements of $f$ are noisy and the data sets always have a bias. Furthermore, computers are discrete beings. Hence, we analyse \cref{eq:inverse_scale_space_contribution} in the following four cases:
\begin{enumerate}
    \item Noiseless and unbiased case; we have access to $f$ and sample from $\rho$.
    \item Noisy case; we have access to $f^\delta$ with measurement noise instead to $f$, but we still want to find to minimizer for $f$.
    \item Biased case; we sample from $\rho^\varepsilon$ with a sampling bias instead of from $\rho$, but we still want to find the minimizer for $\rho$.
    \item Discretized case; the parameter space $\Omega$ is discretized and no longer continuous.
\end{enumerate}
The first shows how well \cref{eq:inverse_scale_space_contribution} can be when we manage to reduce noise and sampling bias to a minimum. The second shows how the methods deals with noise on the data function $f$. The third provides a novel perspective on learning methods. It shows how well the method deals with a bias in the sampling. In machine learning there is a large focus on computing the generalisation error of a method, i.e. how large is the error you make when you solve \cref{eq:contribution_minimisation} with only $n$ samples of $\rho$ relative to using $\rho$ in its entirety. This is one way of having a bias in the sampling. Another bias that one could have as the goal to classify animals based on images to determine whether they are suitable pets, but one has no images of fish. Our method captures both of these biases in one go. The last shows that the method behaves nicely when the parameter space $\Omega$ is discretized.

We show in \cref{sec:main} that the \cref{eq:inverse_scale_space_contribution} is well-defined and determine its optimality conditions. After that we discuss the aforementioned four cases in sections \ref{sec:ideal} to \ref{sec:discretisation} respectively.

\subsection{Background information}
This section provides the relevant background information needed of Barron spaces and Bregman iterations. 

\subsubsection{Barron spaces}
Fix $d\in\N$ and $\sigma$ as an element of $\mathcal{C}^{0,1}(\R)$ or the $\operatorname{ReLU}$ activation function $\max(0,x)$. Let $\X\subseteq\R^d$ and $\Omega\subseteq \R^d\cross\R$. Consider a probability measure $\rho\in \PC(\X)$, and define
\begin{equation}\label{eq:barron_operator}
    K\mu(x)= \int_\Omega \sigma(a^\intercal x +b)d\mu(a,b).
\end{equation}  
for $\mu \in \MC(\Omega)$. \textit{Barron space} $\B_\sigma$ is the Banach space with functions of the form $f=K\mu$ for some $\mu\in\MC(\Omega)$ and 
\begin{equation}
    \norm{f}_{\B_\sigma} = \begin{cases}
    \inf_{K\mu=f}\int_\Omega(1+\norm{w}+\abs{b})d\abs{\mu}(w,b) & \sigma\in \mathcal{C}^{0,1}(\R)\\
    \inf_{K\mu=f}\int_\Omega(\norm{w}+\abs{b})d\abs{\mu}(w,b) & \sigma(x)=\operatorname{ReLU}(x)
    \end{cases}
\end{equation}
The functions in Barron space can be seen as infinitely wide or continuous versions of shallow neural networks
\begin{equation}
    f: \X\to\R, \;x \mapsto \sum_{i=1}^m c_i \sigma(a_i^\intercal x + b_i)
\end{equation}
with $c_i\in \R$ and $(a_i,b_i)\in\Omega$ \parencite{e_barron_2021}. Two embeddings are relevant for this work. They show that Barron functions are nice enough to enable proper convergence.

\begin{proposition}[Barron is Lipschitz; \cite{e_representation_2020}, theorem 3.3]\label{prop:barron_lipschitz_embedding}
If $\rho\in \PC_1(\X)$ is a probability measure with finite first moments, then we have $Lip(f)\leq Lip(\sigma)\norm{f}_{\B_\sigma}$ for every $f\in \B_\sigma$.
\end{proposition}

\begin{proposition}[Barron $L^p$ embedding; \cite{e_representation_2020}, theorem 3.7]\label{prop:barron_Lp_embedding}
If $\rho\in \PC_q(\X)$ is a probability measure with finite $q^{\text{th}}$ moments, then $\B_\sigma \hookrightarrow L^p(\X,\rho)$ for all $1\leq p\leq q$.
\end{proposition}

\subsubsection{Bregman iterations}
Let $\HS$ be some Banach space, $\U$ be a (closed subset of a) thereof, $f\in\HS$, $J: \U\to\R$ be convex, lower semi-continuous and coercive, and $\RC_f: \U\to \R$ be convex, bounded from below and Fréchet differentiable. The Bregman divergence\footnote{The Bregman divergence is often called the Bregman distance, but it is in general neither symmetric nor does it satisfy the triangle inequality.} between $u,v\in\HS$ for $p\in \partial J(v)$ is given by
\begin{equation}
    D^p_J(u,v) = J(u)-J(v) - \braket{p}{u-v}.
\end{equation}
The Bregman iterations 
\begin{equation}\label{eq:bregman_iterations}
    \begin{aligned}
        u_k &= \argmin_{u\in\U} D^{p_{k-1}}_J(u,u_{k-1})+\lambda\RC_f(u) & u_0 = 0\\
        p_k &= p_{k-1} - \lambda\partial_u \RC_f(u_k) & p_0 = 0, p_k\in\partial J(u_k)
    \end{aligned}
\end{equation}
with design parameter $\lambda>0$ are an iterative $5$-approximation algorithm for the bilevel minimization problem
\begin{equation}
\begin{aligned}
    u^\dagger \in &\argmin_{u\in\U}J(u) \\
    \text{s.t. }&u\in\argmin_{\Bar{u}\in\U}\RC_f(\Bar{u}).
\end{aligned}
\end{equation}
The Bregman iterations converge monotonically to the optimal solution with worst case $O(\frac{1}{k})$ convergence \parencite{burger_error_2007}. 

The inverse scale space flow can be derived from \cref{eq:bregman_iterations} by taking the limit of $\lambda \searrow 0$. Before taking the limit, observe that \cref{eq:bregman_iterations} is equivalent to 
\begin{subequations}\label{eq:bregman_iterations_regularized}
    \begin{align}
        u_k &= \argmin_{u\in\U\cap \partial J^\ast(p_k)} \frac{1}{\lambda}\bigg(J(u)-\braket{p_{k-1}}{u}\bigg)+\RC_f(u) & u_0 = 0\label{eq:bregman_iterations_regularized_a}\\
        \frac{p_k-p_{k-1}}{\lambda} &=  - \partial_u\RC_f(u_k) & p_0 =0 \label{eq:bregman_iterations_regularized_b}
    \end{align}
\end{subequations}
Note, that usually \cref{eq:bregman_iterations_regularized_b} has the subgradient constraint $ p_k\in \partial J(\mu_k)$ instead of \cref{eq:bregman_iterations_regularized_a} having $\partial J^\ast(p_k)$ as additional constraint. These two ways of writing the constraint are equivalent by Fenchel duality. In the limit of $\lambda \searrow 0$, \cref{eq:bregman_iterations_regularized_b} can be seen as the Euler discretization of the flow equation
\begin{equation}
    \partial_tp_t = -\partial_u\RC_f(u_t), \quad p_0=0,
\end{equation}
and \cref{eq:bregman_iterations_regularized_a} will find a $u_k$ minimizing $\RC_f(u)$ whilst enforcing that $p_t\in \partial J(u_t)$ or equivalently $u_t\in \partial J^\ast(p_t)$ \parencite{burger_nonlinear_2006}. The inverse scale space is exactly this limit of $\lambda \searrow 0$ of the Bregman iterations, i.e. the dynamical process given by
\begin{subequations}
\label{eq:inverse_scale_space_final}
    \begin{align}
        u_t &= \argmin_{u\in\U \cap \partial J^\ast(p_t)} \RC_f(u) & u_0 = 0,\\
        \partial_tp_t &= -\partial_u\RC_f(u_t) & p_0 = 0.
    \end{align}
\end{subequations}

\subsection{Notation and definitions}
Let $\R$ denote the real numbers, and $\N$ denote the natural numbers without $0$. The space of all Radon measures---regular, signed Borel measures with bounded total variation---on a locally compact Hausdorff $\Omega$ is denoted by $\MC(\Omega)$. It is a Banach space with the norm
\begin{equation*}
    \norm{\mu}_{\MC(\Omega)} = \int_\Omega d\abs{\mu}(x),
\end{equation*}
where $\abs{\mu}$ is the total variation measure of $\mu$. When $\Omega$ is compact and $\MC(\Omega)$ is equipped with the weak*-topology, then $\MC(\Omega)$ is  dual to $\mathcal{C}^0(\Omega)$, the space of continuous functions on $\Omega$. When $\Omega$ is unbounded, then it is dual to $C_0^0(\Omega)$, the space of continuous functions on $\Omega$ that go to zero at infinity. All Radon measures $\mu\in\MC(\Omega)$ have a polar decomposition, i.e. there exists a $\sgn{\mu}\in L^1(\Omega,\abs{\mu})$ with $\abs{\sgn{\mu}}\leq 1$ such that
\begin{equation*}
    d\mu(x) = \sgn{\mu}(x)d\abs{\mu}(x).
\end{equation*}
The space of all probability measures on a set $U$ with finite $k^{\text{th}}$ moments is denoted by $\PC_k(U)\subseteq \MC(U)$. The Wasserstein-1 metric between two probability measures $\rho,\pi\in\PC_1(\Omega)$, can be computed by
\begin{equation*}
    W_1(\rho,\pi) = \sup\Set[\bigg]{ \int_\Omega f(\omega)d\rho(\omega) - \int_\Omega f(\omega)d\pi(\omega) \given f\in \mathcal{C}^0(\Omega), Lip(f)\leq 1},
\end{equation*}
where $Lip(f)$ denotes the Lipschitz constant of $f$. Given a set $X$, a positive number $p\in[1,\infty)$ and a radon measure $\rho\in\MC(X)$, we write $L^p(\rho)$ instead of $L^p(X,\rho)$. If $U\subset V$ is a convex set, $V$ is a locally convex space and $J: U\to\R$ is a convex function, then the convex conjugate is written as $J^*$ and the subgradient $\partial J$ of $J$ at $u_0$ is given by
\begin{equation*}
    \partial J(u_0) = \Set[\bigg]{ v\in V^\ast \given J(u)-J(u_0) \geq \braket{v}{u-u_0}_{V^\ast} \quad\forall u\in U}.
\end{equation*}
(Fréchet) derivatives of a function or operator $f$ are also denoted $\partial f$. If the derivative is a partial derivative, then a subscript will be added to indicate the variable with which the derivative is taken.

\section{Inverse scale space flow for Barron spaces}\label{sec:main}
In this section, we start by defining the necessary functionals and operators to write down the inverse scale space flow for Barron spaces. In \cref{sec:derivation}, we show how to get from the general form of the inverse scale space in \cref{eq:inverse_scale_space_final} to \cref{eq:inverse_scale_space_barron}. Then, in \cref{sec:existence}, we show that this flow is well-defined. Last, in \cref{sec:optimality}, we derive several optimality conditions for the flow that are needed for the proofs of the convergence rates later in this work.

Fix $d\in\N$. Let $\X\subseteq \R^{d}$ and $\Omega\subseteq \R^{d+1}$, $\rho\in \PC_2(\X)$ be a probability measure with bounded second moment, $\sigma\in \mathcal{C}^{0,1}(\R)$ or $\sigma(x)=\max(0,x)$, $V(a,b)=1+\norm{a}+\abs{b}$ and $f\in L^2(\rho)$, where we mean that $a\in\R^d$ and $b\in\R$ when we write $(a,b)\in\Omega$. Use these to define the operators  
\begin{subequations}\label{eq:operators}
    \begin{align}
    K&: \MC(\Omega) \to L^2(\X,\rho),\; \mu \mapsto \bigg(x \mapsto \int_\Omega \sigma(a^\intercal x+b)d\mu(a,b)\bigg) \\
    L_\rho&: L^2(\X,\rho) \to C(\Omega), \; \phi \mapsto  \bigg((a,b)\mapsto \int_\X \phi(x)\sigma(a^\intercal x+b)d\rho(x)\bigg) \\
    J&: \MC(\Omega) \to [0,\infty), \; \mu \mapsto \int_\Omega V(a,b)d\abs{\mu}(a,b) \\
    \RC_f&: \MC(\Omega) \to [0,\infty), \; \mu \mapsto \frac{1}{2}\norm{K\mu -f}_{L^2(\X,\rho)}^2
    \end{align}
\end{subequations}
We consider the task of finding 
\begin{subequations}
\label{eq:optimization_problem}
\begin{align}
    \mu^\text{opt} \in &\argmin_{\mu^\dagger\in\MC(\Omega)}J(\mu^\dagger) \label{eq:optimization_problem_a}\\
    \text{s.t. }&\mu^\dagger\in\argmin_{\mu\in\MC(\Omega)}\RC_f(\mu) \label{eq:optimization_problem_b}
\end{align}
\end{subequations}
The constraint in \cref{eq:optimization_problem_b} says that we are looking for a measure $\mu$ such that $K\mu$ represents the $L^2(\rho)$ projection of $f$ onto Barron space, and \cref{eq:optimization_problem_a} highlights that we want the measure that induces the Barron norm. We will search for the measure $\mu^\text{opt}$ using the inverse scale space flow. The flow corresponding to \cref{eq:optimization_problem} is given by
\begin{subequations}\label{eq:inverse_scale_space_barron}
    \begin{align}
        \mu_t &= \argmin_{u\in \partial J^\ast(p_t)} \RC_f(\mu) & u_0 = 0,\label{eq:inverse_scale_space_barron_a}\\
        \partial_tp_t &= L_{\rho}(f-K\mu_t) & p_0 = 0\label{eq:inverse_scale_space_barron_b}.
    \end{align}
\end{subequations}
In the following, we will assume that every $\mu^\dagger$ we refer to has $J(\mu^\dagger)$ finite.

\subsection{Derivation of the inverse scale space flow for Barron spaces}\label{sec:derivation}
To derive the inverse scale space flow for Barron spaces, we start with \cref{eq:bregman_iterations} and \cref{eq:inverse_scale_space_final}. These imply that the Bregman iterations and associated inverse scale space flow for \cref{eq:optimization_problem} are given by the iterative process 
\begin{subequations}\label{eq:bregman_iterations_barron}
    \begin{align}
        \mu_k &= \argmin_{u\in\MC(\Omega)} D^{p_{k-1}}_J(\mu,\mu_{k-1})+\lambda\RC_f(\mu) & \mu_0 = 0\\
        p_k &= p_{k-1} - \lambda\partial_\mu \RC_f(\mu_k) & p_0 = 0, p_k=\partial J(\mu_k)
    \end{align}
\end{subequations}
and the dynamical system
\begin{subequations}\label{eq:inverse_scale_space_barron_pre}
    \begin{align}
        \mu_t &= \argmin_{\mu\in\MC(\Omega)\cap \partial J^\ast(p_t)} \RC_f(\mu) & \mu_0 = 0,\label{eq:inverse_scale_space_barron_pre_a}\\
        \partial_tp_t &= -\partial_\mu \RC_f(\mu_t) & p_0 = 0\label{eq:inverse_scale_space_barron_pre_b},
    \end{align}
\end{subequations}
respectively. First, observe that $\partial J^\ast(p_t)\subseteq \MC(\Omega)$. This shows that \cref{eq:inverse_scale_space_barron_pre_a} and \cref{eq:inverse_scale_space_barron_a} match. Before we show that \cref{eq:inverse_scale_space_barron_pre_b} is the same as \cref{eq:inverse_scale_space_barron_b}, we show that $L_{\rho}$ is in fact the adjoint of $K$.

\begin{lemma}\label{lemma:adjoint}
The adjoint $L_\rho$ is given by $K$, i.e. $L_\rho^\star=K$. 
\end{lemma}
\begin{proof}
Let $\phi\in L^2(\X,\rho)$ and $\mu\in\MC(\Omega)$, then, by Fubini--Tonelli
\begin{align*}
    \braket{K\mu}{\phi}_{L^2(\rho)} 
    &= \int_\Omega\int_\X \sigma(a^\intercal x+b)d\rho(x)\phi(x)d\mu(a,b)\\
    &= \int_\Omega\int_\X \phi(x)\sigma(a^\intercal x+b)d\rho(x)d\mu(a,b) \\
    &= \braket{\mu}{L_\rho\phi}_{\MC(\Omega)}.
\end{align*}
From the definition of the adjoint it follows that $L_\rho^\star=K$.
\end{proof}
Note that $K$ is the adjoint for all $L_\rho$ with $\rho\in\PC_2(\X)$, but that the difference between the various $L_\rho$ is the inner product used.

\begin{proposition}\label{lemma:variation_derivative}
The variational derivative of $\RC_f$ is given by
\begin{equation}
    \partial_\mu\RC_f(\mu) = L_{\rho}(K\mu-f).
\end{equation}
\end{proposition}
\begin{proof}
Observe that
\begin{align*}
    &\lim_{\norm{\nu}_{\MC(\Omega)}\to0}\frac{\abs{\RC_f(\mu+\nu)-\RC_f(\mu)-\braket{\partial_\mu\RC_f(\mu)}{\nu}_{\MC(\Omega)}}}{\norm{\nu}_{\MC(\Omega)}}
        \\&\quad= \lim_{\norm{\nu}_{\MC(\Omega)}\to0}\frac{\abs{\frac{1}{2}\norm{K(\mu+\nu)-f}_{L^2(\rho)}^2-\frac{1}{2}\norm{K\mu-f}_{L^2(\rho)}^2-\braket{K^\ast(K\mu-f)}{\nu}_{\MC(\Omega)}}}{\norm{\nu}_{\MC(\Omega)}} \\
        &\quad\leq \lim_{\norm{\nu}_{\MC(\Omega)}\to0}\frac{\abs{\frac{1}{2}\norm{K\nu}_{L^2(\rho)}^2-\braket{K\mu-f}{K\nu}_{L^2(\rho)}-\braket{K^\ast(K\mu-f)}{\nu}_{\MC(\Omega)}}}{\norm{\nu}_{\MC(\Omega)}} & \text{triangle ineq.} \\
        &\quad= \lim_{\norm{\nu}_{\MC(\Omega)}\to0}\frac{\abs{\frac{1}{2}\norm{K\nu}_{L^2(\rho)}^2}}{\norm{\nu}_{\MC(\Omega)}} & \text{def. of adjoint} \\
        &\quad\leq \lim_{\norm{\nu}_{\MC(\Omega)}\to0}\frac{1}{2}\norm{K}^2_{op}\norm{\nu}_{\MC(\Omega)} =0.
\end{align*}
Hence,
\begin{equation}\label{eq:partial_Rf_pre_adjoint}
    \partial_\mu\RC_f(\mu) = K^\ast(K\mu-f).
\end{equation}
Combining \cref{lemma:adjoint} with \cref{eq:partial_Rf_pre_adjoint} finishes the proof.
\end{proof}

This shows that \cref{eq:inverse_scale_space_barron_pre_b} is indeed the same as \cref{eq:inverse_scale_space_barron_b}, and thus  that \cref{eq:inverse_scale_space_barron_pre} is the same as \cref{eq:inverse_scale_space_barron}.

\subsection{Existence}\label{sec:existence}
To show that the inverse scale space flow of \cref{eq:inverse_scale_space_barron} has a solution, we use a theorem by Brezis\parencite[theorem 3.1]{brezisOperateursMaximauxMonotones1973}. This theorem establishes that the differential inclusion equation
\begin{equation}\label{eq:brezis}
    \partial_t u_t + Bu_t \in 0
\end{equation}
given some initial condition $u_0\in dom(B):=\Set{u\in H \given Bu \neq \emptyset }$ has a solution. Here, $B$ is a maximally monotone, possibly nonlinear and possibly multivalued function over a Hilbert space $H$. We show that for a suitably chosen maximal operator $B$, the solution to \cref{eq:brezis} exists, and that this solution is in fact a solution to the inverse scale space flow of \cref{eq:inverse_scale_space_barron}.

The operators we need to show that are 
\begin{subequations}\label{eq:brezis_operator}
\begin{align}
    A&: \mathcal{C}(\Omega) \to \MC(\Omega),\; p \mapsto \argmin_{\mu\in\partial\chi_{\Set{\norm{\cdot}_{\infty}\leq 1}}(p)}\RC_f(\mu), \label{eq:brezis_operator_a}\\
    \Tilde{B}&: L^2(\rho) \to L^2(\rho),\; r \mapsto KA(V^{-1}L_\rho r)-f \\
    B&: L^2(\rho) \to L^2(\rho),\; r \mapsto K\partial J^\ast(L_\rho r)-f \label{eq:brezis_operator_c}
\end{align}
\end{subequations}
\begin{lemma}\label{lemma:brezis_form}
The operator $B$ is maximal monotone.
\end{lemma}
\begin{proof}
$J^\ast$ is the Fenchel dual of $J$. Hence, $J^\ast$ is lower semi-continuous, convex and proper. $L_{\rho}$ is a bounded linear operator, so $J^\ast\circ L_\rho$ is also lower semi-continuous, convex and proper. Thus, $r\mapsto\partial J^\ast(L_\rho r)$ is maximal monotone \parencite{brezisMonotoneOperatorsNonlinear1974}. Subtracting a constant from a maximal monotone operator preserves maximal monotonicity, so $B$ is maximal monotone.
\end{proof}

This means the operator $B$ satisfies the requirements for Brezis, and we thus have a solution.
\begin{proposition}\label{prop:brezis}
For every $x\in dom(B)$ there exists a unique function $r:[0,\infty)\to L^2(\rho)$ such that
\begin{enumerate}
    \item $r$ satisfies \cref{eq:brezis} for almost every $t\in (0,\infty)$,
    \item $r_t\in dom(B)$ for all $t>0$,
    \item $r_t$ is Lipschitz continuous on $[0,\infty)$ with $\norm{\partial_t u}_{L^\infty([0,\infty);L^2(\rho)}\leq \norm{B^\circ(x)}$,
    \item $r$ is right differentiable for all $t\in (0,\infty)$ and $\partial^+_tr_t+B^\circ(r_t)=0$ for all $t\in(0,\infty)$,
    \item $t\mapsto B^\circ(r_t)$ is right continuous and $t\mapsto \norm{B^\circ(r_t)}$ non-increasing,
\end{enumerate}
where
\begin{equation}
    B^\circ(r_t) = \argmin_{r\in B(r_t)}\norm{r}_{L^2(\rho)}.
\end{equation}
\end{proposition}
\begin{proof}
See theorem 3.1 of \parencite{brezisOperateursMaximauxMonotones1973}.    
\end{proof}

This does not show that \cref{eq:inverse_scale_space_barron} has a solution yet, since this satisfies \cref{eq:brezis} with the operator $\tilde{B}$ whereas \cref{eq:inverse_scale_space_barron} satisfies \cref{eq:brezis} with the operator $B$.
\begin{lemma}\label{lemma:barron_brezis}
\cref{eq:inverse_scale_space_barron} can be written as 
\begin{equation}\label{eq:barron_brezis}
    \partial_tr_t + B(r_t) = 0,\quad r=0.
\end{equation}
\end{lemma}
\begin{proof}
Substituting \cref{eq:brezis_operator_a} into \cref{eq:inverse_scale_space_barron} gives
\begin{equation}
    \partial_tp_t = L_{\rho}(f-KA(V^{-1}p_t)), \quad p_0=0.
\end{equation}
Replacing $p_t$ with $L_\rho r_t$ gives us
\begin{equation}
    L_{\rho}\partial_tr_t = L_{\rho}(f-KA(V^{-1}L_{\rho}r_t)), \quad r_0=0.
\end{equation}
Since $L_{\rho}$ is a bounded linear operator and thus continuous, $r$ must satisfy
\begin{equation}
    \partial_tr_t = f-KA(V^{-1}L_{\rho}r_t), \quad r_0=0,
\end{equation}
or equivalently
\begin{equation}\label{eq:2_55}
    \partial_tr_t + KA(V^{-1}L_{\rho}r_t) -f = 0, \quad r_0=0.
\end{equation}
Substituting \cref{eq:brezis_operator_c} into \cref{eq:2_55} gives \cref{eq:barron_brezis}.
\end{proof}

To show that there is a solution to \cref{eq:inverse_scale_space_barron}, we use the listed properties of the solution from \cref{prop:brezis}.

\begin{proposition}\label{prop:existence}
\Cref{eq:inverse_scale_space_barron} has a solution for every $\mu_0$ and $p_0$ satisfying $\mu_0=A(V^{-1}L_\rho r_0)$ and $p_0=L_\rho r_0$ for some $r_0\in \dom(B)$. In particular, \cref{eq:inverse_scale_space_barron} has a solution for $\mu_0=0$ and $p_0=0$. 
\end{proposition}
\begin{proof}
Let $r$ be the solution from \cref{prop:brezis} with initial condition $r_0\in\dom(B)$. Since
\begin{equation}
    J^\ast = \chi_{\Set{\norm{V^{-1}\cdot}_{\infty}\leq 1}}
\end{equation}
we have that
\begin{equation}
    B^\circ(r_t) = \argmin_{x\in B(r_t)}\norm{x}_{L^2(\rho)} = K\bigg(\argmin_{\mu\in \partial J^\ast(L_\rho r_t)}\norm{K\mu-f}_{L^2(\rho)}\bigg)-f = KA(V^{-1}L_{\rho}r_t)-f = \Tilde{B}(r_t).
\end{equation}
So in fact, $r$ also solves \cref{eq:brezis} with $\Tilde{B}$, which has the same solution as \cref{eq:inverse_scale_space_barron} by \cref{lemma:barron_brezis} . What remains is to map the solution $r$ to $\mu$ and $p$ using $\mu_t:=A(V^{-1}L_\rho r_t)$ and $p_t:=L_\rho r_t$.
\end{proof}

\begin{remark}
Note that this $\mu_t$ is not unique in general. Since the difference between non-uniqueness is from the null space of $K$, this does not impact any of the later statements.      
\end{remark}

\subsection{Regularity}
The regularity that \cref{prop:brezis} puts on the solution $r$ carries over to $\mu$ and $p$.
\begin{proposition}
$\mu\in L^\infty([0,\infty),\MC(\Omega))$ and $p\in \mathcal{W}^{1,\infty}([0,\infty),\mathcal{C}(\Omega))$.
\end{proposition}
\begin{proof}
Recall from \cref{prop:existence} that $\norm{\partial_tr}_{L^\infty([0,\infty),L^2(\rho))}\leq \norm{f}_{L^2(\rho)}$. This implies that
\begin{equation}\label{eq:r_t_bound}
    \norm{r_t}_{L^2(\rho)} \leq \int_0^t\norm{\partial_sr_s}_{L^2(\rho)}ds \leq t\norm{f}_{L^2(\rho)}.
\end{equation}
We will use this in the norm bounds for both $\mu$ and $p$.

For the regularity of $p$, observe that 
\begin{equation}\label{eq:adjoint_bound}
    \norm{L_\rho}_{L^2(\rho)\to\mathcal{C}(\Omega)} = \norm{K}_{\MC(\Omega)\to L^2(\rho)} < \infty
\end{equation}
by \cref{lemma:adjoint} and \cref{prop:barron_Lp_embedding}.
Since $\partial_t p_t=L_\rho \partial_t r_t$, $p_t=L_\rho r_t$ and $r_t\in L^2(\rho)$, we have
\begin{align}
    \norm{p_t}_{\mathcal{C}(\Omega)} &= \norm{L_\rho r_t }_{\mathcal{C}(\Omega)} \leq \norm{L_\rho}_{L^2(\rho)\to\mathcal{C}(\Omega)}\norm{r_t}_{L^2(\rho)} \leq t\norm{L_\rho}_{L^2(\rho)\to\mathcal{C}(\Omega)}\norm{f}_{L^2(\rho)}, \\
    \norm{\partial_t p_t}_{\mathcal{C}(\Omega)} &= \norm{L_\rho \partial_t r_t }_{\mathcal{C}(\Omega)} \leq \norm{L_\rho}_{L^2(\rho)\to\mathcal{C}(\Omega)}\norm{\partial_t r_t}_{L^2(\rho)} \leq \norm{L_\rho}_{L^2(\rho)\to\mathcal{C}(\Omega)}\norm{f}_{L^2(\rho)}.
\end{align}
by \cref{eq:r_t_bound}, (3) of \cref{prop:existence} and \cref{eq:adjoint_bound}. Hence, $p\in \mathcal{W}^{\infty, 1}([0,T),\mathcal{C}(\Omega))$ with
\begin{equation}
    \norm{p}_{\mathcal{W}^{1,\infty}([0,T),\mathcal{C}(\Omega))} \leq \max(1,t)\norm{L_\rho}_{L^2(\rho)\to\mathcal{C}(\Omega)}\norm{f}_{L^2(\rho)}.
\end{equation}
For the regularity of $\mu$, observe that
\begin{align*}
    \norm{\mu_t}_{\MC(\Omega)} 
    &\leq J(\mu_t) \\
    &= \braket{p_t}{\mu_t}_{\MC(\Omega)} & \text{Fenchel duality} \\
    &= \braket{r_t}{K\mu_t}_{L^2(\rho)} \\
    &= \norm{r_t}_{L^2(\rho)}\norm{K\mu_t}_{L^2(\rho)} & \text{Cauchy-Schwartz}\\
    &= \norm{r_t}_{L^2(\rho)}\norm{K\mu_t-f+f}_{L^2(\rho)} \\
    &\leq \norm{r_t}_{L^2(\rho)}\bigg(\norm{K\mu_t-f}_{L^2(\rho)}+\norm{f}_{L^2(\rho)} \bigg) & \text{triangle ineq.} \\
    &\leq 2\norm{r_t}_{L^2(\rho)}\norm{f}_{L^2(\rho)} \\
    &\leq 2t\norm{f}^2_{L^2(\rho)}. & \tcref{eq:r_t_bound}
\end{align*}
Hence, $\mu\in L^\infty([0,T),\MC(\Omega))$ with
\begin{equation}
    \norm{\mu}_{L^\infty([0,T),\MC(\Omega))} \leq 2T\norm{f}^2_{L^2(\rho)}.
\end{equation}
Since the solution $r$ is unique and the shown regularity holds for all $T>0$, we can extend the regularity to the interval $[0,\infty)$.
\end{proof}

\subsection{Optimality conditions}\label{sec:optimality}
We have now proven the existence and regularity of the solutions to \cref{eq:inverse_scale_space_barron}. In this section, we will have a look at some of the conditions that must hold for the optimal solution. In particular, the orthogonality condition and the source condition. 

We first consider the orthogonality condition. This is a necessary condition, not a sufficient condition.  
\begin{proposition}[Orthogonality condition]\label{prop:ortho_condition}
\begin{equation}
    L_\rho(f-K\mu^\dagger) = 0.
\end{equation}
\end{proposition}
\begin{proof}
For $\mu^\dagger$ to be a minimizer of $\RC_f$, it must hold that
\begin{equation}\label{eq:orth-1}
    \partial_\mu\RC_f(\mu^\dagger) = 0.
\end{equation}
Recall from \cref{lemma:variation_derivative} that
\begin{equation}\label{eq:orth-2}
    \partial_\mu\RC_f(\mu) = L_\rho(f-K\mu).
\end{equation}
Substituting \cref{eq:orth-2} into \cref{eq:orth-1} finishes the proof.
\end{proof}

The second condition we consider is the source condition. This is akin to the existence of a Lagrange multiplier [\cite{burger_convergence_2004}].

\begin{proposition}[Source condition]\label{prop:source_condition}
The source condition is satisfied by $\mu^\dagger$ if there exists a $\phi\in L^2(\X,\rho)$ such that 
\begin{equation}
     L\phi(a,b) = V(a,b)\sgn{\mu^\dagger} \quad \mu^\dagger a.e.
\end{equation}
and 
\begin{equation}
     \abs{L\phi(a,b)} \leq V(a,b)
\end{equation}
for all $(a,b)\in\Omega$.
\end{proposition}
\begin{proof}
We repeat the steps of Bredies in \parencite[around (4.1)]{bredies_inverse_2013}, which in turn in based on \parencite[below def. 1]{burger_convergence_2004}. The source condition is satisfied by $\mu^\dagger$ if there exists a $\phi\in L^2(\X,\rho)$ such that
\begin{equation}\label{eq:source_condition}
     K^{\star}\phi \in \partial\int_\Omega V(a,b)d\abs{\cdot}(\mu^\dagger).
\end{equation}
From the definition of the subdifferential it follows that \cref{eq:source_condition} can only be satisfied when
\begin{equation}\label{eq:subdifferential_condition}
    \braket{K^{\star}\phi}{\nu}_{\MC(\Omega)}-\int_\Omega V(a,b)d\abs{\nu} \leq \braket{K^{\star}\phi}{\mu^\dagger}_{\MC(\Omega)}-\int_\Omega V(a,b)d\abs{\mu^\dagger}
\end{equation}
for all $\nu\in\MC(\Omega)$. Since
\begin{equation}
    \braket{K^{\star}\phi}{\nu}_{\MC(\Omega)} = \braket{\phi}{K\nu}_{L^2(\rho)} = \braket{L_\rho\phi}{\nu}_{\MC(\Omega)}
\end{equation}
by the definition of the adjoint and \cref{lemma:adjoint}, \cref{eq:subdifferential_condition} is equivalent to
\begin{equation}\label{eq:subdifferential_condition_v2}
    \braket{L_\rho\phi}{\nu}_{\MC(\Omega)}-\int_\Omega V(a,b)d\abs{\nu} \leq \braket{L_\rho\phi}{\mu^\dagger}_{\MC(\Omega)}-\int_\Omega V(a,b)d\abs{\mu^\dagger}
\end{equation}
\Cref{eq:subdifferential_condition_v2} must also hold when we take the supremum of the left-hand side.
\begin{equation}\label{eq:subdifferential_condition_v3}
    \sup_{\nu\in \MC(\Omega)}\braket{L_\rho\phi}{\nu}_{\MC(\Omega)}-\int_\Omega V(a,b)d\abs{\nu} \leq \braket{L_\rho\phi}{\mu^\dagger}_{\MC(\Omega)}-\int_\Omega V(a,b)d\abs{\mu^\dagger}
\end{equation}
Every measure $\nu\in\MC(\Omega)$ has a polar decomposition such that
\begin{equation}
    d\nu(a,b) = \sgn{\nu}(a,b)d\abs{\nu}(a,b).
\end{equation}
This allows us to write \cref{eq:subdifferential_condition_v3} as
\begin{equation}\label{eq:subdifferential_condition_v4}
    \sup_{\nu\in \MC(\Omega)}\braket{L_\rho\phi-\sgn{\nu}V}{\nu}_{\MC(\Omega)} \leq \braket{L_\rho\phi\sgn{\mu^\dagger}-V}{\abs{\mu^\dagger}}_{\MC(\Omega)}
\end{equation}
The right-hand side is bounded, so must the left-hand side. If $L_\rho\phi(a,b)>V(a,b)$ for some $(a,b)\in\Omega$, then the left-hand side can be made arbitrarily large by concentrating a large positive $\nu$ around that value. Similarly, if $L_\rho\phi(a,b)<- V(a,b)$ for some $(a,b)\in\Omega$, then the left-hand side can be made arbitrarily large by concentrating a large negative $\nu$ around that value. Hence, $L_\rho\phi$ must satisfy
\begin{equation}\label{eq:L_bound}
    \abs{L_\rho\phi(a,b)} \leq V(a,b).
\end{equation}
Inserting this bound into \cref{eq:subdifferential_condition_v4} gives
\begin{equation}\label{eq:subdifferential_condition_v5}
    0 = \sup_{\nu\in \MC(\Omega)}\braket{L_\rho\phi-\sgn{\nu}V}{\nu}_{\MC(\Omega)} \leq \braket{L_\rho\phi\sgn{\mu^\dagger}-V}{\abs{\mu^\dagger}}_{\MC(\Omega)} \leq 0.
\end{equation}
Hence,
\begin{equation}
    L_\rho\phi=V\sgn{\mu^\dagger}, \quad \mu^\dagger \text{ a.e.}.
\end{equation}
\end{proof}

Note that the source condition described in \cref{prop:source_condition} implies that $\mu_t$ must vanish on the set
\begin{equation}
    \Omega^0_t = \Set[\bigg]{ (a,b)\in\Omega \given -V(a,b) < p_t(a,b) < V(a,b)}.
\end{equation}

\section{Idealized setting}\label{sec:ideal}
In this section, we prove that both the $L^2$ loss $\RC_f(\mu_t)$ and the Bregman distance $D^{p_t}_J(\mu^\dagger,\mu_t)$ decrease monotonically to the optimum value in an ideal setting. The rate at which both of them decrease is of order $\mathcal{O}(1/t)$. This rate is independent of the input dimension $d$.

\begin{theorem}[Ideal case]\label{th:ideal}
$\RC_f(\mu_t)$ is decreasing in time with bound
\begin{equation}
    \RC_f(\mu_t) \leq \RC_f(\mu^\dagger) + \frac{J(\mu^\dagger)}{t}\quad \quad t>0 \text{ a.e.}
\end{equation}
and
\begin{equation}
    \partial_tD^{p_t}_J(\mu^\dagger,\mu_t) \leq 0\quad \quad t\geq 0 \text{ a.e.}
\end{equation}
with equality only when $\mu_t$ minimizes $\RC_f$. Moreover, if $\phi\in L^2(\X,\rho)$ is the function such that the source condition of $\mu^\dagger$ is satisfied, then 
\begin{equation}
    D^{p_t}_J(\mu^\dagger,\mu_t) \leq \frac{\norm{\phi}^2_{L^2(\rho)}}{2t}
\end{equation}
for almost every $t\geq 0$.
\end{theorem}

First, we will show the rate of change of the $L^2$ loss $\RC_f(\mu_t)$ and the Bregman distance $D^{p_t}_J(\mu^\dagger,\mu_t)$ under ideal conditions.

\begin{lemma}\label{lemma:population_loss_grad_neg}
$\RC_f(\mu_t)$ is decreasing in time.
\end{lemma}
\begin{proof}
This follows directly from \cref{prop:brezis} point 5. 
\end{proof}

\begin{lemma}\label{lemma:bregman_dist_grad_neg}
\begin{equation}\label{eq:bregman_dist_grad_neg}
    \partial_tD^{p_t}_J(\mu^\dagger,\mu_t) \leq \RC_f(\mu^\dagger) - \RC_f(\mu_t) 
    \leq 0
\end{equation}
holds for almost every $t\geq 0$.
\end{lemma}
\begin{proof}
This follows from
\begin{align*}
    \partial_t D^{p_t}_J(\mu^\dagger,\mu_t) 
    &= \partial_t\bigg(J(\mu^\dagger)-J(\mu_t)-\braket{p_t}{\mu^\dagger-\mu_t}_{\MC(\Omega)}\bigg) \\
    &= \braket{\partial_tp_t}{\mu_t-\mu^\dagger}_{\MC(\Omega)}-\partial_tJ(\mu_t)+\braket{p_t}{\partial_t\mu_t}_{\MC(\Omega)} \\
    &= \braket{\partial_tp_t}{\mu_t-\mu^\dagger}_{\MC(\Omega)} & p_t\in\partial J(\mu_t) \\
    &\leq \RC_f(\mu^\dagger) - \RC_f(\mu_t) & -\partial_tp_t\in \partial \RC_f(\mu_t) \\
    &\leq 0. & \mu^\dagger\text{ minimizer}
\end{align*}
\end{proof}

\begin{proposition}\label{prop:bregman_dist_grad_neg}
For all $t\geq 0$, it holds that
\begin{equation}
    \partial_tD^{q_t}_J(\mu^\dagger, \mu_t) < 0
\end{equation}
when
\begin{equation}\label{eq:residual_bound}
    \norm{f-K\mu_t}_{L^2(\rho)} > \norm{f-K\mu^\dagger}_{L^2(\rho)}
\end{equation}
as well as when
\begin{equation}\label{eq:residual_bound2}
    \norm{K\mu^\dagger-K\mu_t}_{L^2(\rho)} > 0.
\end{equation}
\end{proposition}
\begin{proof}
\Cref{eq:residual_bound} holds if and only if
\begin{equation}\label{eq:1}
    \RC_f(\mu^\dagger) < \RC_f(\mu_t).
\end{equation}
Recall from the proof of \cref{lemma:bregman_dist_grad_neg} that
\begin{equation}\label{eq:2}
    \partial_t D^{p_t}_J(\mu^\dagger,\mu_t) \leq  \RC_f(\mu^\dagger) - \RC_f(\mu_t).
\end{equation}
The combination of \cref{eq:1} and \cref{eq:2} proves the first statement. For the second statement recall from the proof of \cref{lemma:bregman_dist_grad_neg} that
\begin{equation}\label{eq:breman_diff_MC}
    \partial_tD^{p_t}_J(\mu^\dagger, \mu_t) = \braket{\partial_tp_t}{\mu_t-\mu^\dagger}_{\MC(\Omega)}.
\end{equation}
Hence,
\begin{align*}
    \partial_tD^{p_t}_J(\mu^\dagger, \mu_t)
    &= \braket{\partial_tp_t}{\mu_t-\mu^\dagger}_{\MC(\Omega)} & \tcref{eq:breman_diff_MC} \\ 
    &= \braket{L_\rho(f-K\mu_t)}{\mu_t-\mu^\dagger}_{\MC(\Omega)} & \tcref{eq:inverse_scale_space_barron} \\ 
    &= \braket{L_\rho(f-K\mu_t)-L_\rho(f-K\mu^\dagger)}{\mu_t-\mu^\dagger}_{\MC(\Omega)} & \tcref{prop:ortho_condition} \\ 
    &= \braket{L_\rho(K\mu^\dagger-K\mu_t)}{\mu_t-\mu^\dagger}_{\MC(\Omega)} \\ 
    &= \braket{K\mu^\dagger-K\mu_t}{K\mu_t-K\mu^\dagger}_{L^2(\rho)} & \tcref{lemma:adjoint} \\ 
    &= -\norm{K\mu^\dagger-K\mu_t}^2_{L^2(\rho)}.
\end{align*}
Clearly, this is strictly negative when \cref{eq:residual_bound2} is satisfied.
\end{proof}

\Cref{lemma:bregman_dist_grad_neg} and \cref{lemma:population_loss_grad_neg} show that under ideal conditions the Bregman distance and the population loss respectively are decreasing, and \cref{prop:bregman_dist_grad_neg} shows that this decrease is strict. We will now use these to show that the Bregman distance and the population loss converge and give a rate at which they do that.

\begin{proposition}
If $\mu^\dagger$ satisfies the source condition through $\phi\in L^2(\rho)$, then
\begin{equation}\label{eq:bregman_bound}
    D^{p_t}_J(\mu^\dagger,\mu_t) \leq \frac{\norm{\phi}^2_{L^2(\rho)}}{2t}
\end{equation}
for almost every $t>0$.
\end{proposition}
\begin{proof}
Define 
\begin{equation}\label{eq:residual}
    \partial_t e_t = K\mu^\dagger - K\mu_t, \quad e_0=0
\end{equation}
and 
\begin{equation}\label{eq:p_dagger}
    p^\dagger = L_\rho \phi.
\end{equation}
Observe that
\begin{equation}\label{eq:residual_to_p}
    \partial_t p_t = L_\rho \partial_t e_t, \quad p_0=0=L_\rho e_0.
\end{equation}
With this we obtain
\begin{align*}
    \partial_t\bigg(\frac{1}{2}\norm{e_t-\phi}^2_{L^2(\rho)}\bigg)
     &= \braket{\partial_t e_t}{e_t-\phi}_{L^2(\rho)} \\
     &= \braket{K\mu^\dagger-K\mu_t}{e_t-\phi}_{L^2(\rho)} & \tcref{eq:residual} \\
     &= \braket{L_\rho(e_t-\phi)}{\mu^\dagger-\mu_t}_{\MC(\Omega)} & \tcref{lemma:adjoint} \\
     &= \braket{p_t-p^\dagger}{\mu^\dagger-\mu_t}_{\MC(\Omega)} & \tcref{eq:residual_to_p},\tcref{eq:p_dagger} \\
     &= -\bigg(D^{p_t}(\mu^\dagger,\mu_t)+D^{p^\dagger}(\mu_t,\mu^\dagger)\bigg)
\end{align*}
Hence,
\begin{equation*}
    \partial_t\bigg(\frac{1}{2}\norm{e_t-\phi}^2_{L^2(\rho)}\bigg)+D^{p_t}(\mu^\dagger,\mu_t) \leq 0
\end{equation*}
Integrating from $0$ to $t$ gives
\begin{equation}\label{eq:bregman_diff_integral_bound}
    \int_0^tD^{p_s}(\mu^\dagger,\mu_s)ds+\frac{1}{2}\norm{e_t-\phi}^2_{L^2(\rho)}-\frac{1}{2}\norm{e_0-\phi}^2_{L^2(\rho)} \leq 0.
\end{equation}
Therefore
\begin{align*}
    D^{p_t}(\mu^\dagger,\mu_t) 
        &= \frac{1}{t}\int_0^tD^{p_t}(\mu^\dagger,\mu_t)ds \\
        &= \frac{1}{t}\int_0^tD^{p_s}(\mu^\dagger,\mu_s)ds+\frac{1}{t}\int_0^t\int_s^t\partial_\tau D^{p_\tau}(\mu^\dagger,\mu_\tau)d\tau ds & \text{Fund. th. of calc.}\\ 
        &\leq \frac{1}{t}\int_0^tD^{p_s}(\mu^\dagger,\mu_s)ds & \tcref{lemma:bregman_dist_grad_neg} \\ 
        &\leq -\frac{1}{2t}\norm{e_t-\phi}^2_{L^2(\rho)}+\frac{1}{2t}\norm{e_0-\phi}^2_{L^2(\rho)} & \tcref{eq:bregman_diff_integral_bound} \\
        &\leq \frac{1}{2t}\norm{e_0-\phi}^2_{L^2(\rho)} \\
        &= \frac{1}{2t}\norm{\phi}^2_{L^2(\rho)}. & \tcref{eq:residual_to_p}
\end{align*}
\end{proof}

\begin{proposition}\label{prop:population_loss_ideal_bound}
We have
\begin{equation}    
\RC_f(\mu_t) \leq \RC_f(\mu^\dagger) + \frac{J(\mu^\dagger)}{t}
\end{equation}
for almost every $t>0$.
\end{proposition}
\begin{proof}
Observe that
\begin{align*}
    D^{p_t}_J(\mu^\dagger,\mu_t) -(t-s)\bigg(\RC_f(\mu^\dagger) - \RC_f(\mu_t)\bigg)
    &= D^{p_t}_J(\mu^\dagger,\mu_t) -\int_s^t\bigg(\RC_f(\mu^\dagger) - \RC_f(\mu_t)\bigg)d\tau \\
    &\leq D^{p_t}_J(\mu^\dagger,\mu_t) -\int_s^t\bigg(\RC_f(\mu^\dagger) - \RC_f(\mu_\tau)\bigg)d\tau & \tcref{lemma:population_loss_grad_neg} \\
    &\leq D^{p_t}_J(\mu^\dagger,\mu_t)-\int_s^t\partial_\tau D^{p_\tau}_J(\mu^\dagger,\mu_\tau)d\tau & \tcref{lemma:bregman_dist_grad_neg} \\
    &= D^{p_s}_J(\mu^\dagger,\mu_s). & \text{Fund. th. of calc.}
\end{align*}
Hence, we obtain after rewriting
\begin{align*}
    \RC_f(\mu_t)
    &\leq \RC_f(\mu^\dagger)+\frac{D^{p_s}_J(\mu^\dagger,\mu_s)-D^{p_t}_J(\mu^\dagger,\mu_t)}{t-s}\\
    &\leq \RC_f(\mu^\dagger)+\frac{D^{p_s}_J(\mu^\dagger,\mu_s)}{t-s} & D^{p_t}_J(\mu^\dagger,\mu_t)\geq 0\\
    &\leq \RC_f(\mu^\dagger)+\frac{D^{p_s}_J(\mu^\dagger,\mu_s)}{t} & 0\leq s < t \\
    &\leq \RC_f(\mu^\dagger)+\frac{D^{p_0}_J(\mu^\dagger,\mu_0)}{t} & \tcref{lemma:bregman_dist_grad_neg} \\
    &= \RC_f(\mu^\dagger)+\frac{J(\mu^\dagger)}{t}.
\end{align*}
\end{proof}

\section{Measurement noise}
In this section we prove that with noise on the measurements, the method will converge with $\mathcal{O}(1/t)$ to the solution that best fits the noisy data. If the noise is small enough, then it will at first get closer to the noiseless data, too. After some time, the method will start to get close to the solution for the noisy data and will start moving away from the solution for the noiseless data. The point at which this transition is of the order of the noise, and suggest that the method should be stopped early in the presence of measurement noise.

In the remainder of the work, we consider $f^\delta$ to be some perturbation of $f$ such that
\begin{equation}\label{eq:f_noise_bound}
    \norm{f^\delta-f}^2_{L^2(\rho)} \leq \delta
\end{equation}
with $\delta>0$. When using $f^\delta$ instead of $f$, the flow in \cref{eq:inverse_scale_space_barron} changes. For this section, we will keep referring to the solution based on $f$ with $\mu$ and $p$ whilst we will refer to the solution based on $f^\delta$ with $\nu$ and $q$.

\begin{theorem}[Measurement noise]
We have
\begin{equation}    
    \partial_t D^{p_t}(\mu^\dagger, \nu_t) \leq \frac{\delta^2}{4}, \quad \quad t
    \geq 0 \text{ a.e.}
\end{equation}
and
\begin{equation}
    \partial_tD^{q_t}_J(\mu^\dagger, \nu_t) < 0\quad \quad t
    \geq 0 \text{ a.e.}
\end{equation}
when
\begin{equation}
    \norm{f-K\nu_t}_{L^2(\rho)} > \delta + \norm{f-K\mu^\dagger}_{L^2(\rho)}
\end{equation}
as well as when
\begin{equation}
    \norm{K\mu^\dagger-K\nu_t}_{L^2(\rho)} > \delta.
\end{equation}
Moreover, if $\mu^\dagger$ satisfies the source condition through $\phi\in L^2(\X,\rho)$, then 
\begin{equation}
    D^{q_t}_J(\mu^\dagger, \nu_t) \leq \frac{1}{2t}(\norm{\phi}_{L^2(\rho)}+\delta t)^2+ \frac{\delta^2 t}{8}
\end{equation}
for almost every $t>0$.
\end{theorem}

To prove this, observe that the flow for $f^\delta$ has the same properties as the flow for $f$.

\begin{lemma}\label{lemma:swapping}
$\RC_{f^\delta}(\nu_t)$ is decreasing in $t$.
\end{lemma}
\begin{proof}
Swapping the role of $f$ and $f^\delta$, i.e. considering $f$ to be a perturbation of $f^\delta$, implies that $\RC_{f^\delta}(\nu_t)$ should behave the same as $\RC_{f}(\mu_t)$ from \cref{lemma:population_loss_grad_neg}. Thus, $\RC_{f^\delta}(\nu_t)$ is decreasing in $t$.
\end{proof}

\Cref{lemma:swapping} shows that the inverse scale space converges with $f^\delta$, but it does not tell us how close it will get to the best solution for $f$. 

\begin{lemma}\label{lemma:residual_delta_bound}
\begin{equation}
    \partial_t D^{p_t}(\mu^\dagger, \nu_t) \leq \frac{\delta^2}{4}
\end{equation}
holds for all $t\geq 0$.
\end{lemma}
\begin{proof}
Recall from the proof of \cref{lemma:bregman_dist_grad_neg} that
\begin{equation}\label{eq:breman_diff_MC_2}
    \partial_tD^{q_t}_J(\mu^\dagger, \nu_t) = \braket{\partial_tq_t}{\nu_t-\mu^\dagger}_{\MC(\Omega)}.
\end{equation}
Hence,
\begin{align*}
    \partial_tD^{q_t}_J(\mu^\dagger, \nu_t)
    &= \braket{\partial_tq_t}{\nu_t-\mu^\dagger}_{\MC(\Omega)} & \tcref{eq:breman_diff_MC_2} \\ 
    &= \braket{L(f^\delta-K\nu_t)}{\nu_t-\mu^\dagger}_{\MC(\Omega)} & \tcref{eq:inverse_scale_space_barron_a} \\ 
    &= \braket{L(f^\delta-K\nu_t)-L_\rho(f-K\mu^\dagger)}{\nu_t-\mu^\dagger}_{\MC(\Omega)} & \tcref{prop:ortho_condition} \\ 
    &= \braket{f^\delta-f+K\mu^\dagger-K\nu_t}{K\nu_t-K\mu^\dagger}_{L^2(\rho)} & \tcref{lemma:adjoint} \\ 
    &= \braket{f^\delta-f}{K\nu_t-K\mu^\dagger}_{L^2(\rho)}- \braket{K\nu_t-K\mu^\dagger}{K\nu_t-K\mu^\dagger}_{L^2(\rho)} \\ 
    &\leq \norm{f^\delta-f}_{L^2(\rho)}\norm{K\nu_t-K\mu^\dagger}_{L^2(\rho)}-\norm{K\nu_t-K\mu^\dagger}^2_{L^2(\rho)} & \text{Cauchy Schwartz} \\ 
    &\leq  \frac{1}{4}\norm{f^\delta-f}^2_{L^2(\rho)} & \text{Young's product ineq.} \\ 
    &\leq \frac{\delta^2}{4}.
\end{align*}
\end{proof}

\begin{proposition}\label{prop:residual_delta_bound}
We have
\begin{equation}
    \partial_tD^{q_t}_J(\mu^\dagger, \nu_t) < 0
\end{equation}
for all $t\geq 0$, when
\begin{equation}\label{eq:residual_delta_bound}
    \norm{f^\delta-K\nu_t}_{L^2(\rho)} > \delta + \norm{f-K\mu^\dagger}_{L^2(\rho)}
\end{equation}
as well as when
\begin{equation}\label{eq:residual_delta_bound2}
    \norm{K\mu^\dagger-K\nu_t}_{L^2(\rho)} > \delta.
\end{equation}
\end{proposition}
\begin{proof}
For the first statement observe that
\begin{align*}
    \partial_tD^{q_t}_J(\mu^\dagger, \nu_t)
    &= \braket{\partial_tq_t}{\nu_t-\mu^\dagger}_{\MC(\Omega)} & \tcref{eq:breman_diff_MC} \\ 
    &= \braket{L(f^\delta-K\nu_t)}{\nu_t-\mu^\dagger}_{\MC(\Omega)} & \tcref{eq:inverse_scale_space_barron_a} \\ 
    &= \braket{f^\delta-K\nu_t}{K\nu_t-K\mu^\dagger}_{L^2(\rho)} & \tcref{lemma:adjoint} \\ 
    &= \braket{f^\delta-K\nu_t}{K\nu_t-f^\delta+f^\delta-f+f-K\mu^\dagger}_{L^2(\rho)}  \\     
    &= -\norm{f^\delta-K\nu_t}^2_{L^2(\rho)}+\braket{f^\delta-K\nu_t}{f^\delta-f+f-K\mu^\dagger}_{L^2(\rho)}\\ 
    &\leq -\norm{f^\delta-K\nu_t}^2_{L^2(\rho)}+\norm{f^\delta-f+f-K\mu^\dagger}_{L^2(\rho)}\norm{K\nu_t-K\mu^\dagger}_{L^2(\rho)} & \text{Cauchy Schwartz} \\ 
    &\leq -\norm{f^\delta-K\nu_t}^2_{L^2(\rho)}+\left(\delta+\norm{f-K\mu^\dagger}_{L^2(\rho)}\right)\norm{f^\delta-K\nu_t}_{L^2(\rho)}. & \text{triangle ineq.}, \tcref{eq:f_noise_bound} 
\end{align*}
Clearly, this is strictly negative when \cref{eq:residual_delta_bound} is satisfied. 

For the second statement recall from the proof of \cref{lemma:residual_delta_bound} that
\begin{align*}
    \partial_tD^{q_t}_J(\mu^\dagger, \nu_t)
    &\leq \norm{f^\delta-f}_{L^2(\rho)}\norm{K\nu_t-K\mu^\dagger}_{L^2(\rho)}-\norm{K\nu_t-K\mu^\dagger}^2_{L^2(\rho)}
\end{align*}
Clearly, this is strictly negative when \cref{eq:residual_delta_bound2} is satisfied. 
\end{proof}

From \cref{prop:residual_delta_bound} and \cref{lemma:residual_delta_bound} it follows that the Bregman distance $D^{q_t}_J(\mu^\dagger, \nu_t)$ is guaranteed to converge until $\RC_{f^\delta}(\nu_t)$ is close to $\RC_f(\mu^\dagger)$. We know from \cref{lemma:swapping} that $\RC_{f^\delta}(\nu_t)$ will go to a minimum of $\RC_{f^\delta}$. So we expect the Bregman distance $D^{q_t}_J(\mu^\dagger, \nu_t)$, unlike the Bregman distance $D^{q_t}_J(\mu^\dagger, \mu_t)$, to not go to zero. The following proposition exemplifies this.

\begin{proposition}\label{prop:perturb_data_function}
If $\mu^\dagger$ satisfies the source condition through $\phi\in L^2(\rho)$, then
\begin{equation}\label{eq:perturb_data_function}
    D^{p_t}_J(\mu^\dagger, \nu_t) \leq \frac{1}{2t}\bigg(\norm{\phi}_{L^2(\rho)}+\delta t\bigg)^2+ \frac{\delta^2 t}{8}
\end{equation}
for almost every $t\geq 0$.
\end{proposition}
\begin{proof}
Define 
\begin{equation}\label{eq:R_97}
    \partial_t e_t = f^\delta - K\nu_t +K\mu^\dagger-f, \quad e_0 = 0.
\end{equation}
Observe that
\begin{equation}\label{eq:R_97_to_p}
    \partial_t q_t = L_\rho \partial_t e_t, \quad q_0=0=L_\rho e_0.
\end{equation}
Using this definition of $e_t$ we obtain
\begin{align*}
    \partial_t \bigg(\frac{1}{2}\norm{e_t-\phi}^2_{L^2(\rho)}\bigg) 
    &= \braket{\partial_t e_t}{e_t-\phi}_{L^2(\rho)} \\
    &= \braket{f^\delta - K\nu_t+K\mu^\dagger -f}{e_t-\phi}_{L^2(\rho)} & \tcref{eq:R_97} \\
    &= \braket{f^\delta-f}{e_t-\phi}_{L^2(\rho)}+\braket{K\mu^\dagger-K\nu_t}{e_t-\phi}_{L^2(\rho)} \\
    &\leq \norm{f^\delta-f}_{L^2(\rho)}\norm{e_t-\phi}_{L^2(\rho)}+\braket{K\mu^\dagger-K\nu_t}{e_t-\phi}_{L^2(\rho)} & \text{Cauchy-Schwartz} \\
    &\leq \delta\norm{e_t-\phi}_{L^2(\rho)}+\braket{K\mu^\dagger-K\nu_t}{e_t-\phi}_{L^2(\rho)} & \tcref{eq:f_noise_bound}  \\
    &= \delta\norm{e_t-\phi}_{L^2(\rho)}+\braket{L_\rho(e_t-\phi)}{\mu^\dagger-\nu_t}_{\MC(\Omega)} & \tcref{lemma:adjoint} \\
    &= \delta\norm{e_t-\phi}_{L^2(\rho)}+\braket{q_t-p^\dagger}{\mu^\dagger-\nu_t}_{\MC(\Omega)} & \tcref{eq:R_97_to_p}, p^\dagger:=L_\rho(\phi) \\
    &= \delta\norm{e_t-\phi}_{L^2(\rho)}-\braket{q_t-p^\dagger}{\nu^\dagger-\mu_t}_{\MC(\Omega)}
\end{align*}
Since
\begin{equation}\label{eq:bregman_symm}
   0 \leq D^{p_t}_J(\mu^\dagger, \nu_t) + D^{p^\dagger}_J(\nu_t, \mu^\dagger) = \braket{q_t-p^\dagger}{\nu_t-\mu^\dagger}_{\MC(\Omega)},
\end{equation}
where the inequality stems from that $q_t$ and $p^\dagger$ are from the subgradients $\partial J(\nu_t)$ and $\partial J(\mu^\dagger)$ respectively, we obtain
\begin{equation}
      \partial_t \bigg(\frac{1}{2}\norm{e_t-\phi}^2_{L^2(\rho)}\bigg) \leq \delta\norm{e_t-\phi}_{L^2(\rho)}.
\end{equation}
Solving this for $\norm{e_t-\phi}_{L^2(\rho)}$ gives
\begin{equation}
    \norm{e_t-\phi}_{L^2(\rho)} \leq \norm{e_0-\phi}_{L^2(\rho)} + \delta t = \norm{\phi}_{L^2(\rho)} + \delta t.
\end{equation}
Hence,
\begin{align*}
    \partial_t \bigg(\frac{1}{2}\norm{e_t-\phi}^2_{L^2(\rho)}\bigg)+ D^{p_t}_J(\mu^\dagger, \nu_t) 
    &\leq \delta\norm{e_t-\phi}_{L^2(\rho)}-D^{p^\dagger}_J(\nu_t, \mu^\dagger) \\
    &\leq \delta \norm{\phi}_{L^2(\rho)} + \delta^2 t.
\end{align*}
By integrating both sides of the equation, we obtain
\begin{equation}\label{eq:integrated_bregman_bound}
\begin{aligned}
    \int_0^t D^{p_s}_J(\mu^\dagger, \nu_s) ds + \frac{1}{2}\norm{e_t-\phi}^2_{L^2(\rho)} 
    &\leq \frac{1}{2}\norm{\phi}^2_{L^2(\rho)}+ \delta\norm{\phi}_{L^2(\rho)} t + \frac{1}{2}\delta^2 t^2 \\
    &= \frac{1}{2}\left(\norm{\phi}^2_{L^2(\rho)}+\delta t\right)^2.
\end{aligned}
\end{equation}
Therefore,
\begin{align*}
    D^{p_t}_J(\mu^\dagger, \mu_t) 
    &= \frac{1}{t}\int_0^tD^{p_t}_J(\mu^\dagger, \mu_t)ds \\
    &= \frac{1}{t}\int_0^tD^{s_t}_J(\mu^\dagger, \mu_s)+\int_s^t\partial_\tau D^{p_\tau}_J(\mu^\dagger, \mu_\tau)d\tau ds & \text{Fund. th. of calc.} \\
    &\leq \frac{1}{t}\int_0^tD^{s_t}_J(\mu^\dagger, \mu_s)+\frac{\delta^2}{4}\int_s^td\tau ds & \tcref{lemma:residual_delta_bound} \\
    &= \frac{1}{t}\int_0^tD^{s_t}_J(\mu^\dagger, \nu_s)+\frac{\delta^2}{4}(t-s) ds \\
    &= \frac{1}{t}\int_0^tD^{s_t}_J(\mu^\dagger, \nu_s)ds+\frac{\delta^2}{8}t \\
    &\leq \frac{1}{t}\bigg(\frac{1}{2}\left(\norm{\phi}^2_{L^2(\rho)}+\delta t\right)^2-\norm{e_t-\phi}^2_{L^2(\rho)}\bigg)+\frac{\delta^2}{8}t & \tcref{eq:integrated_bregman_bound} \\
    &\leq  \frac{1}{2t}(\norm{\phi}_{L^2(\rho)}+ \delta t)^2+ \frac{\delta^2}{8}t.
\end{align*}
\end{proof}
\Cref{prop:perturb_data_function} shows us that we should not continue to $t\to\infty$, but should stop earlier. In particular, the bound for \cref{eq:perturb_data_function} is lowest for $t(\delta) = O(\delta^{-1})$.

\section{Biased sampling}
In this section, we prove that a bias in the sampling gives a similar behaviour as noisy measurements. However, the terms and bounds differ depending on how the biased sampling is expressed. We consider sampling expressed in terms of a condition on either the Radon-Nikodym derivative or the Wasserstein-1 distance.  

For the remainder of this work, we consider $\rho^\varepsilon\in \PC_2(\X)$ to be some perturbation of the true distribution $\rho\in \PC_2(\X)$, also with bounded second moment. We assume that $f\in L^2(\rho)\cap L^2(\rho^\varepsilon)$. For this section, we will keep referring to the solution based on $\rho$ with $\mu$ and $p$ whilst we will refer to the solution based on $\rho^\epsilon$ with $\nu$ and $q$. We will also assume that every $\nu^\dagger$ we refer to has $J(\nu^\dagger)$ finite.

\begin{theorem}[Biased sampling of \texorpdfstring{$\rho$}{ρ} -- Radon Nikodym]\label{th:radon_nikodym}
If $\rho^\varepsilon\ll\rho$ and
\begin{equation}
    \norm{1-\dv{\rho^\varepsilon}{\rho}}_{L^\infty(\rho)} \leq \varepsilon,
\end{equation}
then
\begin{equation}
    \partial_t D^{p_t}(\mu^\dagger, \nu_t) < 0
\end{equation}
when
\begin{equation}
    \norm{f-K\nu_t}_{L^2(\rho^\varepsilon)} > (1+\varepsilon)\norm{f-K\mu^\dagger}_{L^2(\rho)}.
\end{equation}
Moreover, if $\mu^\dagger$ and $\nu^\dagger$ satisfy the source condition through $\phi\in L^2(\rho)$ and $\phi\in L^2(\rho^\varepsilon)$ respectively, then
\begin{equation}
\begin{aligned}
    D^{p_t}_J(\mu^\dagger,\mu_t) &\leq \frac{1}{2t}\norm{\phi}^2_{L^2(\rho)}+\frac{\varepsilon}{1+\varepsilon}\frac{1}{2t}\int_0^t\int_0^\tau\norm{K\nu_\tau-K\nu_s}^2_{L^2(\rho^\varepsilon)}dsd\tau \\
    &\quad+(2\varepsilon+1)\frac{t}{4}\norm{f-K\mu^\dagger}^2_{L^2(\rho)} +\frac{t}{4}\norm{f-K\nu^\dagger}^2_{L^2(\rho^\varepsilon)}
\end{aligned}
\end{equation}
for almost every $t \geq 0$.\end{theorem}
\begin{theorem}[Biased sampling of \texorpdfstring{$\rho$}{ρ} -- Wasserstein]\label{th:wasserstein}
If $f\in \mathcal{C}^{0,1}(\supp(\rho-\rho^\varepsilon))$ and
\begin{equation}
    W_1(\rho,\rho^\varepsilon) \leq \varepsilon,
\end{equation}
then
\begin{equation}
    \partial_t D^{p_t}(\mu^\dagger, \nu_t) < 0
\end{equation}
when
\begin{equation}
    \norm{f-K\nu_t}^2_{L^2(\rho^\varepsilon)} > 2\varepsilon\norm{f-K\mu^\dagger}^2_{\mathcal{C}^{(0,1)}(\supp(\rho-\rho^\varepsilon))} + \norm{f-K\mu^\dagger}_{L^2(\rho)}^2.
\end{equation}
Moreover, if $\mu^\dagger$ and $\nu^\dagger$ satisfy the source condition through $\phi\in L^2(\rho)$ and $\phi\in L^2(\rho^\varepsilon)$ respectively, then
\begin{equation}
\begin{aligned}
    D^{p_t}_J(\mu^\dagger,\nu_t) &\leq \frac{1}{2t}\norm{\phi}^2_{L^2(\rho)}+\varepsilon\frac{t}{2}\norm{f-K\mu^\dagger}^2_{\mathcal{C}^{0,1}(\supp(\rho-\rho^\varepsilon))} \\
    &\quad +\frac{\varepsilon}{t}\int_0^t\int_0^\tau\norm{K\nu_\tau-K\nu_s}^2_{\mathcal{C}^{0,1}(\supp(\rho-\rho^\varepsilon))}dsd\tau + \frac{t}{4}\norm{K\nu^\dagger-K\mu^\dagger}^2_{L^2(\rho^\varepsilon)}
\end{aligned}
\end{equation}
for almost every $t\geq 0$.

\end{theorem}

\Cref{th:radon_nikodym} refers to the Radon-Nikodym derivative condition, whereas \cref{th:wasserstein} refers to the Wasserstein-1 distance condition. To prove these theorems, we first consider a general disturbance with no particular conditions on the perturbation $\rho^\varepsilon$. Afterwards, we refine the statements from the general disturbance under the two mentioned conditions in \cref{sec:radon-nikodym,sec:wasserstein}. 

\begin{lemma}\label{lemma:bregman_dist_grad_neg_rho_perp}
We have
\begin{equation}    \partial_tD^{q_t}_J(\mu^\dagger, \nu_t) \leq \frac{1}{4}\norm{f-K\mu^\dagger}^2_{L^2(\rho^\varepsilon)}
\end{equation}
as well as
\begin{equation}
    \partial_tD^{q_t}_J(\mu^\dagger, \nu_t) \leq \frac{1}{4}\norm{K\nu^\dagger-K\mu^\dagger}^2_{L^2(\rho^\varepsilon)}
\end{equation}
for almost every $t\geq 0$.
\end{lemma}
\begin{proof}
The first statement follows from
\begin{align*}
    \partial_tD^{q_t}_J(\mu^\dagger, \nu_t)
    &= \braket{\partial_tq_t}{\nu_t-\mu^\dagger}_{\MC(\Omega)} & \tcref{eq:breman_diff_MC} \\
    &= \braket{L_{\rho^\varepsilon}(f-K\nu_t)}{\nu_t-\mu^\dagger}_{\MC(\Omega)} & \tcref{eq:inverse_scale_space_barron} \\
    &= \braket{f-K\nu_t}{K(\nu_t-\mu^\dagger)}_{L^2(\rho^\varepsilon)} & \tcref{lemma:adjoint} \\
    &= \braket{f-K\nu_t}{K\nu_t-f+f-K\mu^\dagger}_{L^2(\rho^\varepsilon)} \\
    &= -\norm{f-K\nu_t}^2_{L^2(\rho^\varepsilon)}+\braket{f-K\mu^\dagger}{f-K\nu_t}_{L^2(\rho^\varepsilon)} \\
    &\leq -\norm{f-K\nu_t}^2_{L^2(\rho^\varepsilon)}+\norm{f-K\mu^\dagger}_{L^2(\rho^\varepsilon)}\norm{f-K\nu_t}_{L^2(\rho^\varepsilon)} & \text{Cauchy Schwartz} \\
    &\leq \frac{1}{4}\norm{f-K\mu^\dagger}^2_{L^2(\rho^\varepsilon)}. & \text{Young's product ineq.}
\end{align*}
The second statement follows from
\begin{align*}
    \partial_tD^{q_t}_J(\mu^\dagger, \nu_t)
    &= \braket{\partial_tq_t}{\nu_t-\mu^\dagger}_{\MC(\Omega)} & \tcref{eq:breman_diff_MC} \\
    &= \braket{L_{\rho^\varepsilon}(f-K\nu_t)}{\nu_t-\mu^\dagger}_{\MC(\Omega)} & \tcref{eq:inverse_scale_space_barron} \\
    &= \braket{-L_{\rho^\varepsilon}(f-K\nu^\dagger)+L_{\rho^\varepsilon}(f-K\nu_t)}{\nu_t-\mu^\dagger}_{\MC(\Omega)} & \tcref{prop:ortho_condition} \\
    &= \braket{L_{\rho^\varepsilon}(K\nu^\dagger-K\nu_t)}{\nu_t-\mu^\dagger}_{\MC(\Omega)} & \\
    &= \braket{K\nu^\dagger-K\nu_t}{K(\nu_t-\mu^\dagger)}_{L^2(\rho^\varepsilon)} & \tcref{lemma:adjoint} \\
    &= \braket{K\nu^\dagger-K\nu_t}{K\nu_t-K\nu^\dagger+K\nu^\dagger-K\mu^\dagger}_{L^2(\rho^\varepsilon)} \\
    &= -\norm{K\nu^\dagger-K\nu_t}^2_{L^2(\rho^\varepsilon)}+\braket{K\nu^\dagger-K\mu^\dagger}{K\nu^\dagger-K\nu_t}_{L^2(\rho^\varepsilon)} \\
    &\leq -\norm{K\nu^\dagger-K\nu_t}^2_{L^2(\rho^\varepsilon)}+\norm{K\nu^\dagger-K\mu^\dagger}_{L^2(\rho^\varepsilon)}\norm{K\nu^\dagger-K\nu_t}_{L^2(\rho^\varepsilon)} & \text{Cauchy Schwartz} \\
    &\leq \frac{1}{4}\norm{K\nu^\dagger-K\mu^\dagger}^2_{L^2(\rho^\varepsilon)}. & \text{Young's product ineq.}
\end{align*}
\end{proof}

\begin{proposition}\label{prop:bregman_dist_grad_neg_rho_perp}
We have
\begin{equation}    \partial_tD^{q_t}_J(\mu^\dagger, \nu_t) < 0
\end{equation}
when
\begin{equation}\label{eq:epsilon_bound_condition}
    \norm{f-K\nu_t}_{L^2(\rho^\varepsilon)} > \norm{f-K\mu^\dagger}_{L^2(\rho^\varepsilon)}.
\end{equation}
\end{proposition}
\begin{proof}
Recall from the proof of \cref{lemma:bregman_dist_grad_neg_rho_perp} that
\begin{equation}\label{eq:107}
    \partial_tD^{q_t}_J(\mu^\dagger, \nu_t) \leq -\norm{f-K\nu_t}^2_{L^2(\rho^\varepsilon)}+\norm{f-K\mu^\dagger}_{L^2(\rho^\varepsilon)}\norm{f-K\nu_t}_{L^2(\rho^\varepsilon)}.
\end{equation}
Clearly, this is strictly negative when \cref{eq:epsilon_bound_condition} is satisfied.
\end{proof}

\Cref{lemma:bregman_dist_grad_neg_rho_perp} and \cref{prop:bregman_dist_grad_neg_rho_perp} tell us, just like \cref{lemma:swapping} for the noisy case, and as intuitively expected, that the flow will converge until the solution matches the residual. This, however, does not tell us how well it approximates the residual on $\rho$. We will refine this when we consider the more specific disturbances.

We will now provide an upper bound for the Bregman distance.

\begin{proposition}\label{prop:bregman_bound_rho_perp_general}
If $\mu^\dagger$ and $\nu^\dagger$ satisfy the source condition through $\phi\in L^2(\rho)$ and $\phi\in L^2(\rho^\varepsilon)$ respectively, then
\begin{equation}\label{eq:bregman_bound_rho_perp_general}
\begin{aligned}
    D^{p_t}_J(\mu^\dagger,\nu_t) &\leq \frac{1}{2t}\norm{\phi}^2_{L^2(\rho)}+\frac{1}{2t}\int_0^t\int_0^\tau\norm{K\nu_\tau-K\nu_s}^2_{L^2(\rho^\varepsilon-\rho)}dsd\tau \\
    &\quad+\frac{t}{4}\norm{f-K\mu^\dagger}^2_{L^2(\rho^\varepsilon-\rho)} +\frac{t}{8}\norm{K\nu^\dagger-K\mu^\dagger}^2_{L^2(\rho^\varepsilon)}
\end{aligned}
\end{equation}
for almost every $t\geq 0$.
\end{proposition}

\begin{proof}
Define 
\begin{equation}\label{eq:R_109}
    \partial_t e_t = K\mu^\dagger - K\nu_t, \quad e_0 = 0.
\end{equation}
and
\begin{equation}\label{eq:p_dagger_2}
    p^\dagger = L_\rho \phi. 
\end{equation}
With this we obtain
\begin{align*}
    \partial_t\bigg(\frac{1}{2}\norm{e_t-\phi}^2_{L^2(\rho)}\bigg)
     &= \braket{\partial_t e_t}{e_t-\phi}_{L^2(\rho)} \\
     &= \braket{K\mu^\dagger-K\nu_t}{e_t-\phi}_{L^2(\rho)} & \tcref{eq:R_109} \\
     &= \braket{L_{\rho}(e_t-\phi)}{\mu^\dagger-\nu_t}_{\MC(\Omega)} &\tcref{lemma:adjoint} \\
     &= \braket{L_{\rho}(e_t-\phi)-q_t+q_t}{\mu^\dagger-\nu_t}_{\MC(\Omega)} \\
     &= \braket{q_t-L_{\rho}\phi+L_{\rho}e_t-q_t}{\mu^\dagger-\nu_t}_{\MC(\Omega)} \\
     &= \braket{q_t-p^\dagger}{\mu^\dagger-\nu_t}_{\MC(\Omega)}+\braket{L_{\rho}e_t-q_t}{\mu^\dagger-\nu_t}_{\MC(\Omega)} & \tcref{eq:p_dagger_2} \\
     &= -\bigg(D^{q_t}(\mu^\dagger,\nu_t)+D^{p^\dagger}(\nu_t,\mu^\dagger)\bigg)+\braket{L_{\rho}e_t-q_t}{\mu^\dagger-\nu_t}_{\MC(\Omega)}. & \tcref{eq:bregman_symm}
\end{align*}
The rightmost term can be bounded by
\begin{align*}
     &\braket{L_{\rho}e_t-q_t}{\mu^\dagger-\nu_t}_{\MC(\Omega)} \\
     &= \int_0^t\braket{\partial_s(L_{\rho}e_s-q_s)}{\mu^\dagger-\nu_t}_{\MC(\Omega)}ds & \text{Fund. th. of calc.} \\
     &= \int_0^t\braket{L_{\rho}(K\mu^\dagger-K\nu_s)-L_{\rho^\varepsilon}(f-K\nu_s)}{\mu^\dagger-\nu_t}_{\MC(\Omega)}ds & \tcref{eq:R_109} \\
     &= \int_0^t\braket{L_{\rho}(f-K\nu_s)-L_{\rho^\varepsilon}(f-K\nu_s)}{\mu^\dagger-\nu_t}_{\MC(\Omega)}ds & \tcref{prop:ortho_condition} \\
     &= \int_0^t\braket{L_{\rho-\rho^\varepsilon}(f-K\nu_s)}{\mu^\dagger-\nu_t}_{\MC(\Omega)}ds \\
     &= \int_0^t\braket{f-K\nu_s}{K\mu^\dagger-K\nu_t}_{L^2(\rho-\rho^\varepsilon)}ds &\tcref{lemma:adjoint} \\
     &= \int_0^t\braket{f-K\nu_s}{K\nu_t-K\mu^\dagger}_{L^2(\rho^\varepsilon-\rho)}ds \\
     &= \int_0^t\braket{f-K\mu^\dagger-K\nu_t+K\mu^\dagger+K\nu_t-K\nu_s}{K\nu_t-K\mu^\dagger}_{L^2(\rho^\varepsilon-\rho)}ds \\
     &= \int_0^t\braket{f-K\mu^\dagger+K\nu_t-K\nu_s}{K\nu_t-K\mu^\dagger}_{L^2(\rho^\varepsilon-\rho)} \\&\quad- \norm{K\nu_t-K\mu^\dagger}^2_{L^2(\rho^\varepsilon-\rho)}ds \\
     &= \int_0^t\norm{f-K\mu^\dagger+K\nu_t-K\nu_s}_{L^2(\rho^\varepsilon-\rho)}\norm{K\nu_t-K\mu^\dagger}_{L^2(\rho^\varepsilon-\rho)} \\&\quad- \norm{K\nu_t-K\mu^\dagger}^2_{L^2(\rho^\varepsilon-\rho)}ds & \text{Cauchy Schwartz}\\
     &= \int_0^t\bigg(\norm{f-K\mu^\dagger}_{L^2(\rho^\varepsilon-\rho)}+\norm{K\nu_t-K\nu_s}_{L^2(\rho^\varepsilon-\rho)}\bigg)\norm{K\nu_t-K\mu^\dagger}_{L^2(\rho^\varepsilon-\rho)} \\&\quad- \norm{K\nu_t-K\mu^\dagger}^2_{L^2(\rho^\varepsilon-\rho)}ds & \text{Triangle ineq.}\\
     &\leq \frac{1}{2}\int_0^t\norm{f-K\mu^\dagger}^2_{L^2(\rho^\varepsilon-\rho)}ds+\frac{1}{2}\int_0^t\norm{K\nu_t-K\nu_s}^2_{L^2(\rho^\varepsilon-\rho)}ds & \text{Young's prod. ineq.} \\
     &= \frac{t}{2}\norm{f-K\mu^\dagger}^2_{L^2(\rho^\varepsilon-\rho)}+\frac{1}{2}\int_0^t\norm{K\nu_t-K\nu_s}^2_{L^2(\rho^\varepsilon-\rho)}ds.
\end{align*}
Hence,
\begin{equation}
    \partial_t\bigg(\frac{1}{2}\norm{e_t-\phi}^2_{L^2(\rho)}\bigg)+D^{p_t}(\mu^\dagger,\nu_t) \leq \frac{t}{2}\norm{f-K\mu^\dagger}^2_{L^2(\rho^\varepsilon-\rho)}+\frac{1}{2}\int_0^t\norm{K\nu_t-K\nu_s}^2_{L^2(\rho^\varepsilon-\rho)}ds.
\end{equation}
Integrating from $0$ to $t$ gives
\begin{equation}\label{eq:bregman_int_bound_rho_perp}
    \int_0^tD^{p_s}(\mu^\dagger,\nu_s)ds \leq \frac{1}{2}\norm{\phi}^2_{L^2(\rho)}+\frac{t^2}{4}\norm{f-K\mu^\dagger}^2_{L^2(\rho^\varepsilon-\rho)}+\frac{1}{2}\int_0^t\int_0^\tau\norm{K\nu_\tau-K\nu_s}^2_{L^2(\rho^\varepsilon-\rho)}dsd\tau.
\end{equation}
Therefore, we obtain
\begin{align*}
    D^{p_t}(\mu^\dagger,\nu_t) 
        &= \frac{1}{t}\int_0^tD^{p_t}(\mu^\dagger,\nu_t)ds \\
        &\leq \frac{1}{t}\int_0^tD^{p_s}(\mu^\dagger,\nu_s)+\int_s^t\partial_\tau D^{p_\tau}(\mu^\dagger,\nu_\tau)d\tau ds & \text{Fund. th. of calc.} \\
        &\leq \frac{1}{t}\int_0^tD^{p_s}(\mu^\dagger,\nu_s)+\frac{1}{4}\norm{K\nu^\dagger-K\mu^\dagger}^2_{L^2(\rho^\varepsilon)}\int_s^td\tau ds & \tcref{lemma:bregman_dist_grad_neg_rho_perp} \\
        &= \frac{1}{t}\int_0^tD^{p_s}(\mu^\dagger,\nu_s)+\frac{1}{4}\norm{K\nu^\dagger-K\mu^\dagger}^2_{L^2(\rho^\varepsilon)}(t-s) ds  \\
        &= \frac{1}{t}\int_0^tD^{p_s}(\mu^\dagger,\nu_s)ds+\frac{t}{8}\norm{K\nu^\dagger-K\mu^\dagger}^2_{L^2(\rho^\varepsilon)}  \\
        &\leq \frac{1}{2t}\norm{\phi}^2_{L^2(\rho)}+\frac{1}{2t}\int_0^t\int_0^\tau\norm{K\nu_\tau-K\nu_s}^2_{L^2(\rho^\varepsilon-\rho)}dsd\tau & \tcref{eq:bregman_int_bound_rho_perp} \\
        &\quad+\frac{t}{4}\norm{f-K\mu^\dagger}^2_{L^2(\rho^\varepsilon-\rho)} +\frac{t}{8}\norm{K\nu^\dagger-K\mu^\dagger}^2_{L^2(\rho^\varepsilon)}.
\end{align*}
\end{proof}

The bound of \tcref{eq:bregman_bound_rho_perp_general} in \cref{prop:bregman_bound_rho_perp_general} is similar to that of \cref{eq:perturb_data_function} in \cref{prop:perturb_data_function}. If $\nu_t$ remains constant for all $t$ after some time $T\geq0$, then
\begin{equation}
    \frac{1}{2t}\int_0^t\int_0^\tau\norm{K\nu_\tau-K\nu_s}_{L^2(\rho^\varepsilon-\rho)}dsd\tau = \mathcal{O}(1+\frac{1}{t})
\end{equation}
for all $t\geq T$. This implies that \cref{eq:bregman_bound_rho_perp_general}, just like \cref{eq:perturb_data_function}, has a term that is inversely in time, a term constant in time and a term that is linearly increasing in time. 

\subsection{Radon Nikodym}\label{sec:radon-nikodym}
The first type of disturbances is expressed in terms of a bound on the Radon Nikodym derivative. This allows for going from the norm using one measure to the norm using the other measure by adding a multiplicative constant. 

For this subsection, we refine our definition of $\rho^\varepsilon$ by assuming that $\rho^\varepsilon$ is absolutely continuous with respect to $\rho$ with 
\begin{equation}\label{eq:density_perturbation_1}
    \norm{1-\dv{\rho^\varepsilon}{\rho}}_{L^\infty(\rho)} \leq \varepsilon.
\end{equation}

\begin{lemma}\label{lemma:density_change}
For all $g\in L^2(\rho)$
\begin{align}
    \norm{g}^2_{L^2(\rho-\rho^\varepsilon)} &\leq \varepsilon\norm{g}^2_{L^2(\rho)}\label{eq:density_change_rho_rho_eps_to_rho}, \\
    \norm{g}^2_{L^2(\rho^\varepsilon)} &\leq (1+\varepsilon)\norm{g}^2_{L^2(\rho)}\label{eq:density_change_rho_eps_to_rho},
\end{align}
and for all $g\in L^2(\rho^\varepsilon)$
\begin{equation}
    (1-\varepsilon)\norm{g}^2_{L^2(\rho)} \leq \norm{g}^2_{L^2(\rho^\varepsilon)}\label{eq:density_change_rho_to_rho_eps}.
\end{equation}
\end{lemma}
\begin{proof}
The first statement follows from
\begin{align*}
    \norm{g}^2_{L^2(\rho-\rho^\varepsilon)} 
    &= \int_\X g^2(x)d(\rho-\rho^\varepsilon)(x) \\
    &= \int_\X g^2(x)\frac{d(\rho-\rho^\varepsilon)}{d\rho}(x)d\rho(x) \\
    &\leq \norm{1-\dv{\rho^\varepsilon}{\rho}}_{L^\infty(\rho)}\int_\X g^2(x)d\rho(x) \\
    &\leq \varepsilon\norm{g}^2_{L^2(\rho)}.
\end{align*}
For the latter two observe that \cref{eq:density_perturbation_1} means that
\begin{equation}
    1-\varepsilon \leq \dv{\rho^\varepsilon}{\rho} \leq 1+\varepsilon \quad \rho\text{ a.e.}.
\end{equation}
Hence,
\begin{align*}
    \norm{g}^2_{L^2(\rho^\varepsilon)} = \int_\X g^2(x)d\rho^\varepsilon(x) = \int_\X g^2(x)\dv{\rho^\varepsilon}{\rho}{}(x)d\rho(x) \leq (1+\varepsilon)\norm{g}^2_{L^2(\rho)}
\end{align*}
as well as
\begin{align*}
    \norm{g}^2_{L^2(\rho^\varepsilon)} = \int_\X g^2(x)d\rho^\varepsilon(x) = \int_\X g^2(x)\dv{\rho^\varepsilon}{\rho}{}(x)d\rho(x) \geq (1-\varepsilon)\norm{g}^2_{L^2(\rho)}.
\end{align*}
\end{proof}

Using the transformation rules of \cref{lemma:density_change} we can provide conditions on when the rate of change of the Bregman distance is negative, similar to before.

\begin{lemma}
We have
\begin{equation}    
\partial_tD^{q_t}_J(\mu^\dagger, \nu_t) < 0
\end{equation}
for every $t\geq 0$, when
\begin{equation}\label{eq:epsilon_bound_condition_2}
    \norm{f-K\nu_t}_{L^2(\rho^\varepsilon)} > (1+\varepsilon)\norm{f-K\mu^\dagger}_{L^2(\rho)}
\end{equation}
as well as when
\begin{equation}\label{eq:epsilon_bound_condition_3}
    \norm{f-K\nu_t}_{L^2(\rho)} > \frac{1+\varepsilon}{1-\varepsilon}\norm{f-K\mu^\dagger}_{L^2(\rho)}
\end{equation}
and $\varepsilon< 1$.
\end{lemma}
\begin{proof}
Observe that
\begin{align*}
    \partial_tD^{q_t}_J(\mu^\dagger, \nu_t)
    &\leq \norm{f-K\mu^\dagger}_{L^2(\rho^\varepsilon)}\norm{K\nu_t-f}_{L^2(\rho^\varepsilon)}-\norm{K\nu_t-f}^2_{L^2(\rho^\varepsilon)} & \tcref{eq:107} \\
    &\leq (1+\varepsilon)\norm{f-K\mu^\dagger}_{L^2(\rho)}\norm{K\nu_t-f}_{L^2(\rho)}-\norm{K\nu_t-f}^2_{L^2(\rho^\varepsilon)} & \tcref{eq:density_change_rho_eps_to_rho} \\
    &\leq (1+\varepsilon)\norm{f-K\mu^\dagger}_{L^2(\rho)}\norm{K\nu_t-f}_{L^2(\rho)}-(1-\varepsilon)\norm{K\nu_t-f}^2_{L^2(\rho)}. & \tcref{eq:density_change_rho_to_rho_eps}
\end{align*}
Clearly, $\partial_tD^{q_t}_J(\mu^\dagger, \nu_t)$ is strictly negative when either \cref{eq:epsilon_bound_condition_2} or \cref{eq:epsilon_bound_condition_3} is satisfied.
\end{proof}

When comparing \cref{eq:epsilon_bound_condition_2} with \cref{eq:residual_delta_bound}, we see that the sampling bias adds a multiplicative term based on $\varepsilon$. This is unlike the noisy case, where we got an additive term. Likewise, the upper bound for the Bregman distance also gets some multiplicative constants depending on $\varepsilon$. 

\begin{proposition}\label{prop:bregman_distance_bound_radon}
If $\mu^\dagger$ and $\nu^\dagger$ satisfy the source condition through $\phi\in L^2(\rho)$ and $\phi\in L^2(\rho^\varepsilon)$ respectively, then
\begin{equation}\label{eq:bregman_bound_rho_perp_radon_2}
\begin{aligned}
    D^{p_t}_J(\mu^\dagger,\mu_t) &\leq \frac{1}{2t}\norm{\phi}^2_{L^2(\rho)}+\frac{\varepsilon}{1+\varepsilon}\frac{1}{2t}\int_0^t\int_0^\tau\norm{K\nu_\tau-K\nu_s}^2_{L^2(\rho^\varepsilon)}dsd\tau \\
    &\quad+(2\varepsilon+1)\frac{t}{4}\norm{f-K\mu^\dagger}^2_{L^2(\rho)} +\frac{t}{4}\norm{f-K\nu^\dagger}^2_{L^2(\rho^\varepsilon)}
\end{aligned}
\end{equation}
for almost every $t \geq 0$.
\end{proposition}
\begin{proof}
From the transformation rules of \cref{lemma:density_change} it follows that
\begin{equation}\label{eq:other_bounds_radon}
    \frac{t}{4}\norm{f-K\mu^\dagger}^2_{L^2(\rho^\varepsilon-\rho)} \leq \varepsilon\frac{t}{4}\norm{f-K\mu^\dagger}^2_{L^2(\rho)}.
\end{equation}
as well as
\begin{equation}\label{eq:weird_integral_bound_radon}
\begin{aligned}
    \norm{K\nu_\tau-K\nu_s}^2_{L^2(\rho^\varepsilon-\rho)} 
    &= \norm{K\nu_\tau-K\nu_s}^2_{L^2(\rho^\varepsilon)}-\norm{K\nu_\tau-K\nu_s}^2_{L^2(\rho)} \\ 
    &\leq \norm{K\nu_\tau-K\nu_s}^2_{L^2(\rho^\varepsilon)}-\frac{1}{1+\varepsilon}\norm{K\nu_\tau-K\nu_s}^2_{L^2(\rho^\varepsilon)} \\ 
    &= (1-\frac{1}{1+\varepsilon})\norm{K\nu_\tau-K\nu_s}^2_{L^2(\rho^\varepsilon)} \\
    &= \frac{\varepsilon}{1+\varepsilon}\norm{K\nu_\tau-K\nu_s}^2_{L^2(\rho^\varepsilon)}.
\end{aligned}
\end{equation}
Additionally,
\begin{equation}\label{eq:both_dagger_bound_radon}
\begin{aligned}
    &\norm{K\nu^\dagger-K\mu^\dagger}^2_{L^2(\rho^\varepsilon)} \\
    &\quad= \norm{K\nu^\dagger-f+f-K\mu^\dagger}^2_{L^2(\rho^\varepsilon)} \\
    &\quad= \norm{K\nu^\dagger-f}^2_{L^2(\rho^\varepsilon)}+\norm{f-K\mu^\dagger}^2_{L^2(\rho^\varepsilon)}+2\braket{K\nu^\dagger-f}{f-K\mu^\dagger}_{L^2(\rho^\varepsilon)} \\
    &\quad= \norm{K\nu^\dagger-f}^2_{L^2(\rho^\varepsilon)}+\norm{f-K\mu^\dagger}^2_{L^2(\rho^\varepsilon)}+2\norm{K\nu^\dagger-f}_{L^2(\rho^\varepsilon)}\norm{f-K\mu^\dagger}_{L^2(\rho^\varepsilon)} & \text{Cauchy Schwartz} \\
    &\quad= 2\norm{K\nu^\dagger-f}^2_{L^2(\rho^\varepsilon)}+2\norm{f-K\mu^\dagger}^2_{L^2(\rho^\varepsilon)} & \text{Young's product ineq.} \\
    &\quad\leq 2\norm{K\nu^\dagger-f}^2_{L^2(\rho^\varepsilon)}+2(1+\varepsilon)\norm{f-K\mu^\dagger}^2_{L^2(\rho)}. & \tcref{eq:density_change_rho_eps_to_rho} \\
\end{aligned}
\end{equation}
Bounding \cref{eq:bregman_bound_rho_perp_general} using \cref{eq:weird_integral_bound_radon}, \cref{eq:other_bounds_radon} and \cref{eq:both_dagger_bound_radon} gives the sought for expression.
\end{proof}

Note that when we take the limit of $\varepsilon\to 0$ of \cref{eq:bregman_bound_rho_perp_radon_2}, then we get 
\begin{equation}\label{eq:bregman_bound_rho_perp_radon_2_limit}
\begin{aligned}
    D^{p_t}_J(\mu^\dagger,\mu_t) &\leq \frac{1}{2t}\norm{\phi}^2_{L^2(\rho)}+\frac{t}{2}\norm{f-K\mu^\dagger}^2_{L^2(\rho)}.
\end{aligned}
\end{equation}
This shows that the bound for the Bregman distance in \cref{prop:bregman_distance_bound_radon}, unlike the bound in \cref{prop:bregman_bound_rho_perp_general}, is no longer tight in $\varepsilon$. 

An interesting source of bias is when $\rho^\varepsilon$ is a subsampling of $\rho$ such that $\norm{f}_{L^2(\rho^\varepsilon)}$ is a Monte Carlo estimator of $\norm{f}_{L^2(\rho)}$. Clearly, $\rho^\varepsilon\ll \rho$ and $\varepsilon$ is finite. This means that subsampling is a special case of Radon Nikodym bias and that we can use \cref{prop:bregman_distance_bound_radon}. At the same time, the fact that $\norm{f}_{L^2(\rho^\varepsilon)}$ is a Monte Carlo estimator allows us to provide an alternative to \cref{eq:bregman_bound_rho_perp_radon_2}.

\begin{proposition}\label{prop:bregman_distance_bound_sampling}
Let $\rho\in\PC_4(\X)$ be a probability measure with bounded $4^{\text{th}}$ moment, $\rho^\varepsilon$ be a subsampling of $\rho$ with $m(\varepsilon)\in\N$ samples, $\delta>0$, and $f\in L^2(\rho)\cap L^4(\rho)$. If $\mu^\dagger$ and $\nu^\dagger$ satisfy the source condition through $\phi\in L^2(\rho)$ and $\phi\in L^2(\rho^\varepsilon)$ respectively, then
\begin{equation}
\label{eq:bregman_bound_rho_perp_sampling}
\begin{aligned}
    D^{p_t}_J(\mu^\dagger,\nu_t) &\leq \frac{1}{2t}\norm{\phi}^2_{L^2(\rho)}+\frac{1}{2t\sqrt{m(\varepsilon)\delta}}\int_0^t\int_0^\tau\norm{K\nu_\tau-K\nu_s}^2_{L^4(\rho)}dsd\tau \\
    &\quad+\frac{t}{4\sqrt{m(\varepsilon)\delta}}\norm{f-K\mu^\dagger}^2_{L^4(\rho)} +\frac{t}{8}\norm{K\nu^\dagger-K\mu^\dagger}^2_{L^2(\rho^\varepsilon)}.
\end{aligned}
\end{equation}
for almost every $t\geq 0$ with probability at least $1-\delta$.
\end{proposition}
\begin{proof}
Since $\rho$ has bounded $4^{\text{th}}$ moment, we get by \cref{prop:barron_Lp_embedding} that $K\mu\in L^4(\rho)$ for all $\mu\in\MC(\Omega)$.

From Chebychev's inequality it follows that
\begin{align}
    \abs{\norm{K\nu_\tau-K\nu_s}^2_{L^2(\rho^\varepsilon-\rho)}}^2 
    &= \abs{\int_\X \abs{K\nu_\tau(x)-K\nu_s(x)}^2d\rho^\varepsilon(x)-\int_\X \abs{K\nu_\tau(x)-K\nu_s(x)}^2d\rho(x)}^2 \\
    &\leq \frac{\int_\X \abs{K\nu_\tau(x)-K\nu_s(x)}^4d\rho(x)-\bigg(\int_\X \abs{K\nu_\tau(x)-K\nu_s(x)}^2d\rho(x)\bigg)^2}{m(\varepsilon)\delta} \\
    &= \frac{\norm{K\nu_\tau-K\nu_s}^4_{L^4(\rho)}-\norm{K\nu_\tau-K\nu_s}^4_{L^2(\rho)}}{m(\varepsilon)\delta}\\
    &\leq \frac{\norm{K\nu_\tau-K\nu_s}^4_{L^4(\rho)}}{m(\varepsilon)\delta}.
\end{align}
with probability at least $1-\delta$.
Taking the square root on both sides gives
\begin{equation}\label{eq:122}
    \norm{K\nu_\tau-K\nu_s}^2_{L^2(\rho^\varepsilon-\rho)}\leq \frac{\norm{K\nu_\tau-K\nu_s}^2_{L^4(\rho)}}{\sqrt{m(\varepsilon)\delta}}.
\end{equation}
Similarly,
\begin{equation}\label{eq:123}
    \norm{f-K\mu^\dagger}^2_{L^2(\rho^\varepsilon-\rho)}\leq \frac{\norm{f-K\mu^\dagger}^2_{L^4(\rho)}}{\sqrt{m(\varepsilon)\delta}}.
\end{equation}
Substitution of \cref{eq:122} and \cref{eq:123} into \cref{eq:bregman_bound_rho_perp_general} gives \cref{eq:bregman_bound_rho_perp_sampling}.
\end{proof}Note that when we take the limit of $m(\varepsilon)\to \infty$ of \cref{eq:bregman_bound_rho_perp_sampling}, then we get 
\begin{equation}\label{eq:bregman_bound_rho_perp_sampling_limit}
\begin{aligned}
    D^{p_t}_J(\mu^\dagger,\mu_t) &\leq \frac{1}{2t}\norm{\phi}^2_{L^2(\rho)}.
\end{aligned}
\end{equation}
This shows that the bound for the Bregman distance in \cref{prop:bregman_distance_bound_sampling}, like the bound in \cref{prop:bregman_bound_rho_perp_general}, is tight in $\varepsilon$.

\subsection{Wasserstein}\label{sec:wasserstein}
The second type of disturbances is expressed in terms of a bound on the Wasserstein metric. This allows for going from the norm using one measure to
the norm using the other measure by using the duality between Wasserstein and the Lipschitz continuous function with Lipschitz constant at most 1.

For this subsection, we refine our definition of $\rho^\varepsilon$ by assuming that the Wasserstein-1 distance between $\rho^\varepsilon$ and $\rho$ is bounded through $\varepsilon$, i.e.,
\begin{equation}\label{eq:density_perturbation_2}
    W_1(\rho^\varepsilon, \rho)\leq \varepsilon.
\end{equation}
We also assume that $f\in \mathcal{C}^{0,1}(\supp(\rho-\rho^\varepsilon)$.

\begin{lemma}\label{lemma:wasserstein_change_rule}
For all $g\in \mathcal{C}^{0,1}(\supp(\rho-\rho^\varepsilon)$
\begin{equation}
    \norm{g}^2_{L^2(\rho^\varepsilon-\rho)} \leq 2\norm{g}^2_{\mathcal{C}^{0,1}(\supp(\rho^\varepsilon-\rho)}\varepsilon.
\end{equation}
\end{lemma}
\begin{proof}
Recall that
\begin{align*}
    W_1(\rho^\varepsilon, \rho) = \sup_{\substack{h\in \mathcal{C}^{0,1}(\X)\\Lip(h)\leq 1}}\braket{h}{\rho^\varepsilon-\rho}_{\MC(\X)}.
\end{align*}
Since for all $g\in \mathcal{C}^{0,1}(\supp(\rho^\varepsilon-\rho))$
\begin{equation}
    Lip(\frac{g}{Lip(g)}) \leq 1,
\end{equation}
we obtain
\begin{equation}\label{eq:129}
    \braket{g}{\rho^\varepsilon-\rho}_{\MC(\X)} = Lip(g)\braket{\frac{g}{Lip(g)}}{\rho^\varepsilon-\rho}_{\MC(\X)} \leq Lip(g)W_1(\rho^\varepsilon,\rho)\leq Lip(g)\varepsilon,
\end{equation}
where we used \cref{eq:density_perturbation_2}. Furthermore, $Lip(\abs{g}^2)\leq 2\norm{g}^2_{\mathcal{C}^{0,1}(\supp(\rho^\varepsilon-\rho))}<\infty$ since
\begin{equation}\label{eq:130}
    \abs{\abs{g(x)}^2-\abs{g(y)}^2} = \abs{g(x)-g(y)}\abs{\abs{g(x)}+\abs{g(y)}} \leq 2\norm{g}_{\mathcal{C}^0(\supp(\rho^\varepsilon-\rho))}Lip(g)\norm{x-y}_{\ell^\infty}
\end{equation}
for all $x,y\in \supp(\rho^\varepsilon-\rho)$. Hence, 
\begin{align*}
    \norm{g}^2_{L^2(\rho^\varepsilon-\rho)} 
    &= \braket{\abs{g}^2}{\rho^\varepsilon-\rho}_{\MC(\X)} \\
    &\leq Lip(\abs{g}^2)\varepsilon & \tcref{eq:129} \\
    &\leq 2\norm{g}^2_{\mathcal{C}^{0,1}(\X)}\varepsilon.  & \tcref{eq:130}
\end{align*}
for all $g\in \mathcal{C}^{0,1}(\supp(\rho^\varepsilon-\rho))$.
\end{proof}

\begin{proposition}\label{prop:bregman_grad_neg_wasserstein}
We have
\begin{equation}
    \partial_t D^{q_t}(\mu^\dagger, \nu_t) < 0
\end{equation}
when
\begin{equation}\label{eq:something}
    \norm{K\nu_t-f}^2_{L^2(\rho^\varepsilon)} > 2\varepsilon\norm{f-K\mu^\dagger}^2_{\mathcal{C}^{0,1}(\supp(\rho-\rho^\varepsilon)} + \norm{f-K\mu^\dagger}^2_{L^2(\rho)}.
\end{equation}
\end{proposition}
\begin{proof}
$f-K\mu^\dagger$ is a sum of two Lipschitz functions on $\supp(\rho-\rho^\varepsilon)$; $f$ by assumption and $K\mu^\dagger$ by \tcref{prop:barron_lipschitz_embedding}. Thus, $f-K\mu^\dagger$ is Lipschitz on $\supp(\rho-\rho^\varepsilon)$. From \cref{lemma:wasserstein_change_rule} we obtain that
\begin{equation}\label{eq:91}
    \norm{f-K\mu^\dagger}^2_{L^2(\rho^\varepsilon)} \leq 2\varepsilon\norm{f-K\mu^\dagger}^2_{\mathcal{C}^{0,1}(\supp(\rho-\rho^\varepsilon)} + \norm{f-K\mu^\dagger}^2_{L^2(\rho)}.
\end{equation}
Hence,
\begin{align*}
    \partial_tD^{q_t}_J&(\mu^\dagger, \nu_t)\leq \norm{f-K\mu^\dagger}_{L^2(\rho^\varepsilon)}\norm{K\nu_t-f}_{L^2(\rho^\varepsilon)}-\norm{K\nu_t-f}^2_{L^2(\rho^\varepsilon)} & \tcref{eq:107} \\
    &\quad\leq \sqrt{2\varepsilon\norm{f-K\mu^\dagger}^2_{\mathcal{C}^{0,1}(\supp(\rho-\rho^\varepsilon)} + \norm{f-K\mu^\dagger}^2_{L^2(\rho)}}\norm{K\nu_t-f}_{L^2(\rho^\varepsilon)}-\norm{K\nu_t-f}^2_{L^2(\rho^\varepsilon)}. & \tcref{eq:91}
\end{align*}
Clearly, this is strictly negative when \cref{eq:something} is satisfied.
\end{proof}

When comparing \cref{eq:epsilon_bound_condition_2} with \cref{eq:residual_delta_bound}, we see that the sampling bias adds an additive term based on $\varepsilon$. This is like the noisy case, but unlike when the sampling bias was given in terms of the Radon--Nikodym derivative.

\begin{proposition}\label{prop:bregman_distance_bound_wasserstein}
If $\mu^\dagger$ and $\nu^\dagger$ satisfy the source condition through $\phi\in L^2(\rho)$ and $\phi\in L^2(\rho^\varepsilon)$ respectively, then
\begin{equation}\label{eq:bregman_bound_rho_perp_wasserstein}
\begin{aligned}
    D^{p_t}_J(\mu^\dagger,\nu_t) &\leq \frac{1}{2t}\norm{\phi}^2_{L^2(\rho)}+\varepsilon\frac{t}{2}\norm{f-K\mu^\dagger}^2_{\mathcal{C}^{0,1}(\supp(\rho-\rho^\varepsilon))} \\
    &\quad +\frac{\varepsilon}{t}\int_0^t\int_0^\tau\norm{K\nu_\tau-K\nu_s}^2_{\mathcal{C}^{0,1}(\supp(\rho-\rho^\varepsilon))}dsd\tau + \frac{t}{8}\norm{K\nu^\dagger-K\mu^\dagger}^2_{L^2(\rho^\varepsilon)}
\end{aligned}
\end{equation}
for almost every $t\geq 0$.
\end{proposition}
\begin{proof}
Recall from the proof of \cref{prop:bregman_grad_neg_wasserstein} that $f-K\mu^\dagger\in \mathcal{C}^{0,1}(\supp(\rho-\rho^\varepsilon))$. \Cref{eq:91} can be rewritten as 
\begin{equation}\label{eq:wasserstein_bound_1}
    \norm{f-K\mu^\dagger}^2_{L^2(\rho^\varepsilon-\rho)} \leq 2\varepsilon\norm{f-K\mu^\dagger}^2_{\mathcal{C}^{0,1}(\supp(\rho-\rho^\varepsilon)}.
\end{equation}
Similarly, $K\nu_\tau-K\nu_s\in \mathcal{C}^{0,1}(\supp(\rho-\rho^\varepsilon))$ by proposition 1.2. Hence,
\begin{equation}\label{eq:wasserstein_bound_2}
    \norm{K\nu_\tau-K\nu_s}^2_{L^2(\rho^\varepsilon-\rho)} \leq 2\varepsilon\norm{K\nu_\tau-K\nu_s}^2_{\mathcal{C}^{0,1}(\supp(\rho-\rho^\varepsilon)}.
\end{equation}
Bounding \cref{eq:bregman_bound_rho_perp_general} using \cref{eq:wasserstein_bound_1} and \cref{eq:wasserstein_bound_2} gives the sought for expression.
\end{proof}

Note that when we take the limit of $\varepsilon\to 0$ of \cref{eq:bregman_bound_rho_perp_radon_2}, then we get 
\begin{equation}\label{eq:bregman_bound_rho_perp_wasserstein_2_limit}
\begin{aligned}
    D^{p_t}_J(\mu^\dagger,\mu_t) &\leq \frac{1}{2t}\norm{\phi}^2_{L^2(\rho)}.
\end{aligned}
\end{equation}
This shows that the bound for the Bregman distance in \cref{prop:bregman_distance_bound_wasserstein}, like the bound in \cref{prop:bregman_bound_rho_perp_general}, is tight in $\varepsilon$.

\section{Parameter space discretisation}\label{sec:discretisation}
One issue with the inverse scale space of \cref{eq:inverse_scale_space_contribution} is that $p_t$ is defined on $\Omega$. To ensure that $p_t \in \partial J(\mu_t)$ we need to have full knowledge of $p_t$. This cannot be implemented. Hence, $\Omega$ needs to be discretized. In this section, we study a particular discretization based on the Voronoi tessellation. 
In \cref{sec:voronoi}, we show, for a given sequence of Voronoi tessellations with mild assumptions, that the inverse scale space flow on these tessellations converges to the full flow for $N\to\infty$. In \cref{sec:barron_tessellation}, we show the rate of convergence for the flow with fixed $N$ to the optimal solution. Combined, these sections prove \cref{th:discretisation}.

Given a set $\omega^N \subseteq \Omega$ with $\abs{\omega^N}=N$, a Voronoi tessellation divides $\Omega$ into $N$ subsets 
\begin{equation}
    \Omega^N_n = \Set[\bigg]{ w\in\Omega \given \forall m\in\Set{1,\ldots,N}: \;|w-\omega^N_n|\leq |w-\omega^N_m|}
\end{equation}
such that 
\begin{equation}
    \Omega = \bigcup_{n=1}^N \Omega^N_n.
\end{equation}
We consider sequences of sets $\Set{\omega^N}_{N=1}^\infty$ with $\omega^N \subseteq \Omega$, $\abs{\omega^N}=N$ and $\lim_{N\to\infty}\max_{n}\diam(\Omega^N_n) = 0$. For this section, we will keep referring to the solution over $\Omega$ with $\mu$ and $p$ whilst we will refer to the solution over $\omega^N$ with $\nu$ and $q$. With $\nu^\dagger$ we denote a minimizer of $\RC_f$ with $J(\nu^\dagger)<\infty$ over the measures supported on $\omega^N$. We will make use of the Lagrangian 
\begin{equation}
    F: \MC(\Omega) \to [0,\infty), \; \mu \mapsto J(\mu)+\lambda\RC_f(\mu)
\end{equation}
of \cref{eq:inverse_scale_space_barron} and its restriction to $\omega^N$
\begin{equation}
    F_N\mu = \begin{cases}F\mu & \supp(\mu)\subseteq \omega^N \\ \infty & \text{otherwise}\end{cases}
\end{equation}
in the proofs. We will also assume that $\Omega$ is compact.

\begin{theorem}\label{th:discretisation}
The sequence $\Set{F_N}_{N=1}^\infty$ satisfies
\begin{equation}
    F_N\xrightarrow{\Gamma} F
\end{equation}
and its sequence of minimizers converges in weak$^*$ to the minimizer of $F$. Moreover,
\begin{equation}\label{eq:discretization_bound}
    \norm{K\nu_t-f}^2_{L^2(\rho)} \leq 2\norm{K\mu^\dagger-f}^2_{L^2(\rho)} + 2Lip(\sigma)^2(\max_n\diam(\Omega_n^N))^2\norm{\mu^\dagger}^2+2\frac{J(\nu^\dagger)}{t}.
\end{equation}
for almost every $t\geq 0$.
\end{theorem}

\subsection{Convergence of the discrete flow to the full flow}\label{sec:voronoi}
Both the discrete flow and full flow are well-defined flows, so what remains to show is that the solutions to the discrete flow for increasing $N$ converge to the solution for the full flow. To prove this, we will show that the Lagrangian of the discrete flow $F_N$ $\Gamma$-converges to the Lagrangian of the full flow $F$ and that the associated minimizers converge in weak$^\ast$. The requirements for this to hold is that $F_N$ satisfies the $\lim\inf$  property, that there exists a $\Gamma$-realizing sequence and that the family $(F_N)_{N}$ is equicoercive \parencite{braides_handbook_2006}. These three properties are the requirements for the fundamental theorem of $\Gamma$-convergence. The three propositions at the end of this subsection show that these hold. These propositions rely on some properties of $F$ that carry over to $F_N$. We will prove those first.

\begin{lemma}\label{lemma:F_properties}
$F$ is proper, convex, weak* lower semi-continuous and coercive. 
\end{lemma}
\begin{proof}
$F$ is proper, since $0\in dom(F)$.

Since $V$ is continuous, $J$ is convex. Since $K$ is a bounded, linear (and thus continuous) operator and the square of the $L^2(\rho)$ norm is convex, $\RC_f$ is convex. Since $F$ is a sum of two convex functions, $F$ is convex.

Let $(\mu_n)$ be a sequence of measures and $\mu_n,\mu\in\MC(\Omega)$, such that $\mu_n\xrightarrow{w^\star}\mu$. Then for all $\phi\in L^2(\rho)$
\begin{equation}
    \lim_{n\to\infty}\braket{K\mu_n}{\phi}_{L^2(\rho)} = \lim_{n\to\infty}\braket{L_\rho\phi}{\mu_n}_{\MC(\Omega)} = \braket{L_\rho\phi}{\mu}_{\MC(\Omega)} = \braket{K\mu}{\phi}_{L^2(\rho)}.
\end{equation}
This shows that $K\mu_n\xrightarrow{L^2(\rho)}K\mu$. Since
\begin{equation}\label{eq:loss}
    w \mapsto \frac{1}{2}\norm{w-f}^2_{L^2(\rho)}
\end{equation}
is continuous and convex, it is sequentially weak lower-semicontinuous. The combination implies that $\RC_f$ is sequentially weak$^\star$ lower-semicontinuous. Since $J$ is continuous, it is weak$^\star$ lower-semicontinuous. This implies that $F$ is weak$^\star$ lower-semicontinuous.

$F$ is coercive if and only if 
\begin{equation}
    \lim_{\norm{\mu}_{\MC(\Omega)}\to\infty}F(\mu) = \infty.
\end{equation}
For measures $\mu\notin N(K)$ outside the kernel of $K$ we have that $\RC_f(\mu)\to\infty$ as $\norm{\mu}_{\MC(\Omega)}\to\infty$. Since $J$ is non-negative, $F$ will grow without bound for those measures too. What remains is the measures $\mu\in N(K)$ inside the kernel of $K$. For these measures $\RC_f(\mu)$ is constant, but by the conditions on $V$ imply that $J$ will grow without bound. Hence, $F$ is coercive.
\end{proof}

Now, we can prove the three properties needed for the sequence of $F_N$'s.

\begin{proposition}[Liminf property]\label{prop:liminf}
For all $\mu\in\MC(\Omega)$ and every sequence $(\mu_n)$ such that $\mu_n\xrightarrow{w^*}\mu$, we have
\begin{equation}
    \liminf\limits_{\mu_n\to\infty}F_n(\mu_n) \geq F(\mu).
\end{equation}
\end{proposition}
\begin{proof}
From construction of $F_n$ it follows that 
\begin{equation}
    F_n(\mu) \geq F(\mu).
\end{equation}
Hence, combined with the lower semi-continuity of $F$ proven in \cref{lemma:F_properties}, we obtain
\begin{equation}
    \liminf\limits_{\mu_n\to\infty}F_n(\mu_n) \geq \liminf\limits_{\mu_n\to\infty}F(\mu_n) \geq F(\mu).
\end{equation}
\end{proof}

\begin{proposition}[$\Gamma$-realizing sequence]\label{prop:gamma_realizing}
Let $\mu\in\MC(\Omega)$ and define a sequence of measures $\mu_N\in\MC(\Omega)$ by
\begin{equation}
    \mu_N = \sum_{n=1}^N\mu(\Omega^N_n)\delta_{\omega^N_n}.
\end{equation}
We have $\mu_N\xrightarrow{w^*}\mu$ as well as
\begin{equation}
    \lim\limits_{N\to\infty}F_N(\mu_N) = F(\mu).
\end{equation}
\end{proposition}
\begin{proof}
Recall that $\MC(\Omega)$ is dual to $C(\Omega)$, so the weak* convergence is defined in terms of $g\in \mathcal{C}(\Omega)$. Since $\Omega$ is compact, $g$ is absolutely continuous. Recall that this implies that
\begin{equation}
    \forall \varepsilon >0 \exists \delta>0 \forall (a,b),(c,d)\in \Omega: \norm{(a,b)-(c,d)}<\delta \implies \abs{g(a,b)-g(c,d)}< \varepsilon. 
\end{equation}
Since the diameter of the Voronoi cells  vanishes as $N$ goes to infinity, there must be an $\Tilde{N}$ such that for all $N>\Tilde{N}$ and $n\in\Set{1,\hdots,N}$ we have that $\norm{(a,b)-(a^N_n,b^N_n)}<\delta$ for all $(a,b)\in\Omega^N_n$. Hence, for all $g\in \mathcal{C}(\Omega)$ and all $\varepsilon>0$
\begin{align*}
    \lim_{N\to\infty}\abs{\int_\Omega g(a,b) d(\mu-\mu_N)(a,b)}
     &= \lim_{N\to\infty}\abs{\int_\Omega g(a,b) d\bigg(\mu-\sum_{n=1}^N\mu(\Omega^N_n)\delta_{\omega^N_n}\bigg)(a,b)} \\
     &= \lim_{N\to\infty}\abs{\int_\Omega g(a,b) d\mu(a,b)-\sum_{n=1}^Ng(a^N_n,b^N_n)\mu(\Omega^N_n)} \\
     &= \lim_{N\to\infty}\abs{\int_\Omega g(a,b) d\mu(a,b)-\sum_{n=1}^N\int_{\Omega^N_n}g(a^N_n,b^N_n)d\mu(a,b)} \\
     &= \lim_{N\to\infty}\abs{\sum_{n=1}^N\int_{\Omega^N_n} g(a,b) d\mu(a,b)-\sum_{n=1}^N\int_{\Omega^N_n}g(a^N_n,b^N_n)d\mu(a,b)} \\
     &= \lim_{N\to\infty}\abs{\sum_{n=1}^N\int_{\Omega^N_n} \bigg(g(a,b)-g(a^N_n,b^N_n)\bigg) d\mu(a,b)} \\
     &\leq \lim_{N\to\infty}\sum_{n=1}^N\int_{\Omega^N_n} \norm{g(a,b)-g(a^N_n,b^N_n)} d\abs{\mu}(a,b) \\
    &< \lim_{N\to\infty}\sum_{n=1}^N\int_{\Omega^N_n} \varepsilon d\abs{\mu}(a,b) \\
    &= \varepsilon\norm{\mu}_{\MC(\Omega)}
\end{align*}
Since $\varepsilon$ was arbitrary, we must have that
\begin{equation}
    \lim_{N\to\infty}\int_\Omega g(a,b) d(\mu-\mu_N)(a,b) = 0.
\end{equation}
This shows that $\mu_N\xrightarrow{w^*}\mu$, and by construction of $\mu_N$ we have $F_N(\mu_N)=F(\mu_N)$. Furthermore, we showed in \cref{lemma:F_properties} that $F$ was weak$^*$ lower semi-continuous. If fact, by similar arguments, it is sequentially weak$^*$ continuous. Hence, it follows that
\begin{equation}
    \lim_{N\to\infty}F_N(\mu_N) = \lim_{N\to\infty}F(\mu_N) = F(\mu).
\end{equation}
\end{proof}

\begin{proposition}[Equicoercivity]
The family $(F_N)_{N}$ is equicoercive.
\end{proposition}
\begin{proof}
The family $(F_N)_{N}$ is equicoercive if and only if every member of the family is coercive. In \cref{lemma:F_properties} it was proven that $F$ is coercive. Hence, by construction of $F_N$
\begin{equation}
    \lim_{\norm{\mu}_{\MC(\Omega)}\to\infty}F_N(\mu) \geq \lim_{\norm{\mu}_{\MC(\Omega)}\to\infty}F(\mu) = \infty.
\end{equation}
This means that $F_N$ is coercive. Since $N$ was arbitrary, it holds for all members $F_N$ of the family $(F_N)_{N}$.  
\end{proof}

We have now shown that the requirements for the fundamental theorem of $\Gamma$-convergence hold, which implies that $F_N\xrightarrow{\Gamma}F$ and that the sequence of minimizers of $F_N$ converges in weak$^*$ to the minimizer of $F$. 

\subsection{Convergence error for the discrete flow}\label{sec:barron_tessellation}
In the previous section, we showed that the discrete flow converges to the full flow. In this section, we will fix $N$ and show the convergence rates of the discrete flow to the optimal solution. We will first show the generic bound, also shown in \cref{th:discretisation}. Afterward, we will look at a special case.

Observe that the finite $\omega^N$ satisfies the required properties for a proper inverse scale space flow. The following proposition shows the generic bound.

\begin{proposition}
We have
\begin{equation}\label{eq:147}    
\begin{aligned}
    \norm{K\nu_t-f}^2_{L^2(\rho)} &\leq 2\norm{K\mu^\dagger-f}^2_{L^2(\rho)} +2\frac{J(\nu^\dagger)}{t} \\
    &\quad+ 2\norm{\mu^\dagger}^2_{\MC(\Omega)}(\max_n\diam(\Omega^N_n))^2Lip(\sigma)^2\int_\X\max(1,\norm{x})^2d\rho(x).
\end{aligned}
\end{equation}
for almost every $t\geq0$.
\end{proposition}
\begin{proof}
From \cref{prop:population_loss_ideal_bound} it follows that
\begin{equation}\label{eq:149}
    \norm{K\nu_t-f}^2_{L^2(\rho)} \leq \norm{K\nu^\dagger-f}^2_{L^2(\rho)} + 2\frac{J(\nu^\dagger)}{t}.
\end{equation}
Since $\nu^\dagger$ is a minimizer of $\RC_f$ over $\omega^N$, we have for the measure
\begin{equation}
    \mu_N = \sum_{n=1}^N\mu^\dagger(\Omega^N_n)\delta_{\omega^N_n}
\end{equation}
that
\begin{equation}
    \norm{K\nu^\dagger-f}_{L^2(\rho)} \leq \norm{K\mu_N-f}_{L^2(\rho)} \leq \norm{K\mu_N-K\mu^\dagger}_{L^2(\rho)} + \norm{K\mu^\dagger-f}_{L^2(\rho)}
\end{equation}
and thus by Young's inequality for products with $p=q=2$
\begin{equation}\label{eq:151}
    \norm{K\nu^\dagger-f}^2_{L^2(\rho)} \leq 2\norm{K\mu_N-K\mu^\dagger}^2_{L^2(\rho)} + 2\norm{K\mu^\dagger-f}^2_{L^2(\rho)}.
\end{equation}
We observe that by a similar argument as in the proof of \cref{prop:gamma_realizing} that
\begin{align*}
    \norm{K\mu_N-K\mu^\dagger}^2_{L^2(\rho)} &= \int_\X\abs{\int_\Omega \sigma(a^\intercal x+b)d(\mu_n-\mu^\dagger)(a,b)}^2d\rho(x) \\
    &\leq \int_\X\bigg(\sum_{n=1}^N\int_{\Omega^N_n} \norm{\sigma(a^\intercal x+b)-\sigma((a_n^N)^\intercal x+b_n^N)}d\abs{\mu^\dagger}(a,b)\bigg)^2d\rho(x) \\
    &\leq \int_\X\bigg(\sum_{n=1}^N\int_{\Omega^N_n}Lip(\sigma)\norm{(a^\intercal x+b)-((a_n^N)^\intercal x+b_n^N)}d\abs{\mu^\dagger}(a,b)\bigg)^2d\rho(x) \\
    &\leq \int_\X\bigg(\sum_{n=1}^N\int_{\Omega^N_n}Lip(\sigma)\bigg(\norm{a-a_n^N}\norm{x}+\abs{b-b_n^N}\bigg)d\abs{\mu^\dagger}(a,b)\bigg)^2d\rho(x) \\
    &\leq \int_\X\max(1,\norm{x})^2\bigg(\sum_{n=1}^N\int_{\Omega^N_n}Lip(\sigma)\bigg(\norm{a-a_n^N}+\abs{b-b_n^N}\bigg)d\abs{\mu^\dagger}(a,b)\bigg)^2d\rho(x) \\
    &= Lip(\sigma)^2\int_\X\max(1,\norm{x})^2\bigg(\sum_{n=1}^N\int_{\Omega^N_n}\diam(\Omega^N_n)d\abs{\mu^\dagger}(a,b)\bigg)^2d\rho(x) \\
    &\leq \norm{\mu^\dagger}_{\MC(\Omega)}^2(\max_n\diam(\Omega^N_n))^2Lip(\sigma)^2\int_\X\max(1,\norm{x})^2d\rho(x).
\end{align*}
Substituting this into \cref{eq:151} and the resulting expression into \cref{eq:149} gives \cref{eq:147}.
\end{proof}

In \cite{devroye_measure_2015} it was shown that a Voronoi cell's radius decreases with a rate of $O(N^{-1/d})$ when points the points in $\omega^N$ are i.i.d. sampled from an absolutely continuous probability measure over $\Omega$. We can use the direct approximation theorem of Barron spaces to achieve a better rate \parencite[Theorem 3.8]{e_representation_2020}.

\begin{proposition}
Let $N\in\N$. Denote with $M_f$ the set of all measures $\mu_N$ of $N$atoms that satisfy the bounds
\begin{equation}\label{eq:M_f}
    \norm{K\mu_N-K\mu^\dagger}^2_{L^2(\rho)} \leq \frac{J(\mu^\dagger)^2}{N}Lip(\sigma)^2\int_\X\max(1+\norm{x})^2d\rho(x),
\end{equation}
and choose $\omega^N$ such that $M_f$ is non-empty. Then, 
\begin{equation}\label{eq:155}
\begin{aligned}
    \norm{K\nu_t-f}^2_{L^2(\rho)} &\leq 3\norm{K\mu^\dagger-f}^2_{L^2(\rho)}+ 2\frac{J(\nu^\dagger)}{t}\\&\quad+  3\frac{J(\mu^\dagger)^2}{N}Lip(\sigma)^2\int_\X\max(1,\norm{x})^2d\rho(x)  + 3\inf_{\mu_N\in M_f}\norm{K\nu^\dagger-K\mu_N}^2_{L^2(\rho)}.
\end{aligned}
\end{equation}
\end{proposition}
\begin{proof}
$K\mu^\dagger\in\B$, so by \parencite[theorem 4]{e_barron_2021} there exists a suitable choice for $\omega^N$. Let $\mu_N\in M_f$. Observe that
\begin{align*}
    \norm{K\nu_t-f}^2_{L^2(\rho)} 
    \leq \norm{K\nu^\dagger-f}^2_{L^2(\rho)} &+ 2\frac{J(\nu^\dagger)}{t} & \tcref{prop:population_loss_ideal_bound} \\
    \leq 3\norm{K\nu^\dagger-K\mu_N}^2_{L^2(\rho)} &+ 3\norm{K\mu^\dagger-K\mu_N}^2_{L^2(\rho)} + 3\norm{K\mu^\dagger-f}^2_{L^2(\rho)} + 2\frac{J(\nu^\dagger)}{t} & \triangle\text{ ineq., Young's} \\
    \leq 3\norm{K\nu^\dagger-K\mu_N}^2_{L^2(\rho)} &+ 3\frac{J(\mu^\dagger)^2}{N}Lip(\sigma)^2\int_\X\max(1,\norm{x})^2d\rho(x)\\+ 3\norm{K\mu^\dagger-f}^2_{L^2(\rho)} &+ 2\frac{J(\nu^\dagger)}{t}. & \tcref{eq:M_f}
\end{align*}
Taking the infimum over $\mu_N\in M_f$ gives \cref{eq:155}.
\end{proof}

\section{Discussion}
In this work, we have studied the convergence and error analysis of finding the best measure $\mu$ such that the Barron function $K\mu$ is close to $f$ using the inverse scale space flow. After having established the existence and regularity of the solution, we considered the ideal, noisy, biased, and discretized cases. For each of these cases, we analysed the evolution of the Bregman divergence with respect to the optimal solution $D^{p_t}(\mu^\dagger,\nu_t)$ and the $L^2$ loss $\RC_f(\mu_t)$. 

In the ideal case, we got monotonic and linear evolution to the optimal solution. In the noisy case, we still got monotonic and linear evolution to the optimal solution but only up to an error level determined by the noise level $\delta$. These results agree with the known results for inverse scale spaces.

In the novel case of biased sampling, $D^{p_t}(\mu^\dagger,\nu_t)\leq O(1+\frac{1}{t}+t)$ with the suppressed factors in the big O notation depending on $\varepsilon$. When we work with noisy measurements, $D^{p_t}(\mu^\dagger,\nu_t)$ has a similar upper bound but depending on $\delta$. In that setting, the smallest upper bound for $D^{p_t}(\mu^\dagger,\nu_t)$ is attained for $t(\delta)=O(\delta^{-1})$. When dealing with biased sampling, this smallest upper bound is attained for $t(\varepsilon)=O(\frac{\sqrt{1+\varepsilon}}{\sqrt{1+\varepsilon+\varepsilon^2}})$ and $t(\varepsilon)=O(\frac{\sqrt{1+\varepsilon}}{\sqrt{\varepsilon}})$ for a Radon Nikodym and a Wasserstein perturbation respectfully. However, whilst in many cases it is straightforward to provide an estimate for $\delta$, it is not the case for $\varepsilon$.

A second issue with the upper bounds for $D^{p_t}(\mu^\dagger,\nu_t)$ is that we typically do not know $f$, $\phi$, $\mu^\dagger$, $\nu^\dagger$ or $\rho$. What we do know is $K\nu_t$ on $\supp(\rho^\varepsilon)$. This means the bound in \cref{prop:bregman_distance_bound_radon} has more terms that can be explicitly computed than the bounds in \cref{prop:bregman_bound_rho_perp_general}, \cref{prop:bregman_distance_bound_sampling} or \cref{prop:bregman_distance_bound_wasserstein}. That makes \cref{prop:bregman_distance_bound_radon} arguably the most useful proposition.

When the parameter space $\Omega$ is discretized, we have shown that we still have a proper inverse scale space flow. In this setting, we get an additional additive factor depending on $N$ in convergence. When we don't make any additional assumptions on $\omega^N$, this additional factor is of the form $O(N^{-1/d})$. This $1/d$ factor shows that the discretization method suffers from the \textit{curse of dimensionality}, meaning that the method performs poorly when working with high dimension. Although we show that an $O(N^{-1/2})$ can be attained in theory, it is unclear how to find the required $N$ points without solving a different sparse minimization problem first. 

\section*{Acknowledgements}
TJH and CB acknowledge support by Sectorplan Bèta (the Netherlands) under the focus area \enquote{Mathematics of Computational Science}. MB, TR and CB acknowledge support of the European Union's Horizon 2020 research and innovation programme under the Marie Sk{\l}odowska-Curie grant agreement No 777826 (NoMADS). MB and TR further acknowledge support from DESY (Hamburg, Germany), a member of the Helmholtz Association HGF, by the German Ministry of Science and Technology (BMBF) under grant agreement No. 05M2020 (DELETO). MB also acknowledges support  from the German Research Foundation, project BU 2327/19-1. Most of this study was carried out while TR was affiliated with the Friedrich-Alexander-Universität Erlangen-Nürnberg.

\printbibliography

\end{document}